\documentclass[letterpaper, 10 pt, onecolumn]{ieeeconf}  
\usepackage{amsmath,graphicx,multicol}
\usepackage{amsfonts}
\usepackage{color}
\usepackage{bbm}
\usepackage{cite}
\usepackage{amssymb}
\usepackage{algorithm}
\usepackage{algorithmic,booktabs}
\usepackage[margin=1in]{geometry} %this was also included which messed with the formatting 
\usepackage{hyperref}
\usepackage{tabularx}
\usepackage[utf8]{inputenc}
\usepackage[english]{babel}
\usepackage{dsfont}
\usepackage{subcaption}
\usepackage[font={footnotesize}]{caption}
\usepackage{url}
\usepackage{mathtools}
\usepackage{comment}
\usepackage{tcolorbox}
\usepackage{float}
\usepackage{amssymb}
\usepackage{fancyhdr}
\usepackage{cleveref}
\usepackage{listings}
\usepackage{xcolor}
\usepackage{amsthm}

\newtheorem{assumption}[]{Assumption}
\newtheorem{definition}{Definition}[]

\newtheorem{theorem}{Theorem}[]

\theoremstyle{remark}
\newtheorem*{remark}{Remark}

\DeclareMathOperator{\E}{\mathbb{E}}

\begin{document}

\title{\LARGE \textbf{
On the Convergence of
Decentralized Federated Learning Under Imperfect Information Sharing}}
% Impact of Noisy Communication Channels on Convergence in Decentralized Federated Learning}}
\author{
Vishnu Pandi Chellapandi$^*$, Antesh Upadhyay$^*$, Abolfazl Hashemi, and Stanislaw H \.{Z}ak
\thanks{V. P. Chellapandi, A. Upadhyay, A. Hashemi, and S. H. \.{Z}ak are with College of Engineering, Purdue University, West Lafayette, IN 47907, USA. Emails: {\tt\small \{cvp,aantesh,abolfazl,zak\}@purdue.edu}}
\thanks{$^*$ These authors contributed equally to the manuscript}
}

\maketitle
% \thispagestyle{empty}
% \pagestyle{plain}

%%%%%%%%%%%%%%%%%%%%%%%%%%%%%%%%%%%%%%%%%%%%%%%%%%%%%%%%%%%%%%%%%%%%%%%%%%%%%%%%
\begin{abstract}
Decentralized learning and optimization is a central problem in control that encompasses several existing and emerging applications, such as federated learning. While there exists a vast literature on this topic and most methods centered around the celebrated average-consensus paradigm, less attention has been devoted to scenarios where the communication between the agents may be imperfect.
To this end, this paper presents three different algorithms of Decentralized Federated Learning (DFL) in the presence of imperfect information sharing modeled as noisy communication channels. The first algorithm, Federated Noisy Decentralized Learning (FedNDL1), comes from the literature, where the noise is added to their parameters to simulate the scenario of the presence of noisy communication channels. This algorithm shares parameters to form a consensus with the clients based on a communication graph topology through a noisy communication channel. The proposed second algorithm (FedNDL2) is similar to the first algorithm but with added noise to the parameters, and it performs the gossip averaging before the gradient optimization. The proposed third algorithm 
 (FedNDL3), on the other hand, shares the gradients through noisy communication channels instead of the parameters. Theoretical and experimental results demonstrate that under imperfect information sharing, the third scheme that mixes gradients is more robust in the presence of a noisy channel compared with the algorithms from the literature that mix the parameters. 

% \textcolor{red}{abolfazl:1. to save some space, we can merge all three algorithms in one ALG environment and use colors, labels, etc. to refer to each one 
% 2. to save some space for more discussion perhaps we can report the plots on only one dataset/task 
% 3. we also need plots with other network structures, e.g. fully-connected and torus 
% 3. to save some space for more discussion perhaps we can merge all theorems into one with simplified notation
\end{abstract}


%%%%%%%%%%%%%%%%%%%%%%%%%%%%%%%%%%%%%%%%%%%%%%%%%%%%%%%%%%%%%%%%%%%%%%%%%%%%%%%%
\section{Introduction}
Due to technological advances, a massive amount of data is being generated from devices like computers, mobiles, smart watches, and vehicles which are collected in centralized data centers and subsequently used for training machine learning models. However, challenges such as limited communication bandwidth, memory constraints, and privacy concerns make centralized learning not reliable and scalable. Hence, learning paradigms that promote a secure and privacy-preserving environment must be created. This led to an advancement in decentralized optimization algorithms \cite{nedic2009distributed,koloskova2020unified} such as the Decentralized Federated Learning (DFL) \cite{hashemi2021benefits} and Federated Learning (FL)\cite{mcmahan2017communication,konevcny2016federated} where only weights or gradients are transferred instead of the raw data from all agents involved.  FL, in particular, is a new technique developed to address these issues specifically and has found applications in several domains like hospitals, mobile mobiles, and connected vehicles~\cite{yang2019federated,savazzi2021opportunities,zeng2021federated}.

\subsection{Related work}
A canonical approach to decentralized optimization is consensus-based gradient descent methods \cite{nedic2018network,nedic2009distributed,tsitsiklis1984problems,qin2022decentralized} which compute local weights and gradients for all clients' local data and then share the computed parameters with other clients.
%\footnote{In this paper, agents, nodes, and clients are used interchangeably to refer to the participants in a distributed learning paradigm.} 
The weight/gradient is then averaged from all the clients based on a network topology that dictates the communication structure of the learning paradigm. The network topology can be represented as a simple graph where edges represent communication links with individual clients through which the parameters (weights or gradients) are shared. 
% The communication link is assumed to be un-directed and the communication matrix is doubly stochastic.  
These decentralized topologies minimize critical bottlenecks of centralized methods, such as network communication latency and network bandwidth, to improve scalability and efficiency in large-scale settings~\cite{tsitsiklis1984problems,hendrickx2020stability}.

While communication efficiency is one of the critical elements and challenges for distributed learning, and several efforts, including communication compression, have been made in this regard \cite{li2020federated,reisizadeh2020fedpaq,du2020high,zheng2020design,hashemi2021benefits,chen2021communication,chen2021decentralized},
% FL~\cite{ li2020federated}. Efforts have been made to reduce the communication rounds and the amount of data transmitted in each round~\cite{reisizadeh2020fedpaq,du2020high,zheng2020design,hashemi2021benefits,chen2021communication,chen2021decentralized}. FL communication has several challenges, such as unstable network connections between clients compared to traditional centralized learning. Some techniques like sparsification and quantization can compress the model parameters before transmitting the weights/gradients to the central server, thereby minimizing the communication overhead and loss~\cite{ shahid2021communication}. 
these methods generally assume that the communication channels are noiseless. The performance of the trained model in the presence of noise should be one of the critical criteria in choosing the machine learning framework to ensure the robustness and safety of emerging applications that rely on distributed learning.

The effect of imperfect information sharing such as noisy communication  or quantization noise in an average consensus algorithm in a distributed framework was studied in~\cite{carli2007average,qin2021communication}. However, the impact of various levels of noise has not been studied. Additionally, the study in \cite{carli2007average} is limited to consensus problems only and does not encompass unique challenges that arise in modern decentralized optimization and learning, e.g., the inherent non-convexity of the learning objective. Other works including \cite{amiri2020federated,zhu2019broadband,xia2021fast,sery2021over,guo2020analog,wei2022federated} study the impact of noise in server-assisted FL. These works require a server, and they require somewhat restrictive  assumptions that typically are not satisfied in practical settings or are hard to verify. Unlike FL, DFL has no central server, and each client, effectively,  acts as its own individual server. Here, each client, typically, performs local Stochastic Gradient Descent (SGD) or its variations on its local dataset and exchanges messages only with its immediate neighbors. 
%Decentralized SGD have been researched in with a noise model has been studied on smooth and strongly convex functions~\cite{qin2021communication}.

In this paper, our primary focus is on DFL in the presence of noise in communication channels. Recently, \cite{reisizadeh2019exact,vasconcelos2021improved,reisizadeh2022distributed,reisizadeh2022almost} study  the performance of a two-time scale method \cite{srivastava2011distributed} for DFL with channel noise while requiring the convexity of the objective function, uniformly bounded gradients, and access to the deterministic gradients; note that these three considerations are very restrictive assumptions, especially in emerging settings in large-scale learning.

\subsection{Contribution}
Motivated by the existing gap between perfect information and noisy decentralized learning, in this paper, we model the presence of noise in the communication channels as random vector with zero mean and different variances and study the performance of three decentralized FL algorithms by adding the noise to the parameters. 
%\textcolor{red}{abolfazl:I think it is best not to put too much emphasis in the novelty of the algorithms as their variations have been explored in the noiseless case before.}
In particular, we study the impact of noise in communication channels in three algorithms. The first algorithm, Federated Noisy Decentralized Learning (FedNDL1) was recently considered in~\cite{koloskova2020unified}, where the parameters were not subjected to any communication noise. In our analysis of this algorithm, we added noise to the parameters after the local SGD update. The new parameter with the noise is then exchanged with other clients through the gossip matrix (communication matrix as per the communication graph), and the global parameters are updated. This iteration, also known as communication rounds, continues throughout the training. In the second algorithm (FedNDL2), which is related to \cite{nedic2009distributed}, the noise is added before the consensus and local SGD update. In the third algorithm (FedNDL3), which has been considered in the noiseless case in \cite{rabbat2015multi}, the noise is added to the gradients as opposed to the parameters, and the result is exchanged with the immediate clients. 

We demonstrate, theoretically and empirically, that there are benefits in using FedNDL3 in the imperfect information setting that communicates the stochastic gradients. The intuition which is formalized theoretically is that the parameters are  sensitive to the added noise while the stochastic gradients, which are already imperfect, are resilient. Therefore, the error stemming from weaker consensus in FedNDL3 is not as severe as the detrimental impact of noise on FedNDL1 and FedNDL2.

\section{Problem Statement}
In this section, we describe the problem structure, assumptions, and the proposed algorithms that we analyzed in this paper. We start with a standard DFL setup in which $n$ clients/agents have their own local datasets and collaborate with each other to update the global parameters. Formally, the problem can be represented as 
\begin{equation}
    \label{eq:Objective-function}
    \min_{x\in \mathbb{R}^{d}}\Big[f(x)= \frac{1}{n}\sum_{i=1}^{n} f_i(x_i) \Big],
\end{equation}
where $f_i:\mathbb{R}^{d} \to\mathbb{R}$ for $i \in \{1, \dots , n\}$ is the local objective function of the $i^{th}$ client node. The stochastic formulation of the local objective function can be written
\begin{equation}
    \label{eq:local-equation}
    f_i(x) = \E_{\xi_i \sim \mathcal{D}_{i}}[\ell(x_i,\xi_i)],
\end{equation}
where $\xi_i$ is the data that has been sampled from the data distribution $\mathcal{D}_{i}$ for the $i^{th}$ client. The function $\ell(x_i,\xi_i)$ is the loss function evaluated for each client and for each data sample $\xi_i$. Here $x_i \in \mathbb{R}^d$ is the parameter vector of client $i$, and $X \in \mathbb{R}^{d\times n}$ is the matrix formed using these parameter vectors. The $i$-th column of this matrix corresponds to the parameter vector of $i$-th client. Thus, the primary objective of the clients is to achieve optimality through collaboration i.e., $x_i = x^{*}$. 
%for some $x^{*} \in \mathcal{D}^{*}$. 

The main idea of the process is to achieve consensus in which the client can only communicate with its adjacent neighbors. This process of communication can be modeled using a communication graph with the help of a consensus matrix. More precisely, a client $i$ communicates with client $j$ based on a non-negative weight, $w_{ij} > 0$, that formulates the connectivity of client $i$ and client $j$. Similarly, for self-loops, the associated weight, $w_{ii} > 0$, and $w_{ij} = 0$ if there is no communication supposed to happen between $i$ and $j$. These associated weights are then placed in a matrix of dimension $n \times n$ and can be written as $W = [w_{ij}] \in [0,1]^{n\times n}$. The standard name for $W$ in the literature is the gossip or mixing matrix. To proceed, we define the mixing matrix.
\begin{definition}[\textbf{Mixing matrix}]
    The mixing/gossip matrix, $W = [w_{ij}] \in [0,1]^{n\times n}$, is a non-negative, symmetric $(W = W^{\top})$ and doubly stochastic $(\mathds{1}W = \mathds{1}, \mathds{1}^{\top}W = \mathds{1}^{\top})$ matrix, where $\mathds{1}$ is the column vector of unit elements of size $n$
\end{definition}

% \textcolor{red}{abolfazl: in the above definition you have a fixed $W$ but in the subsequent appearance of $W$ you assume it is sampled from a family of mixing matrices. Which is correct? which have you used in the proofs? it seems to me that you have used a fixed $W$.}
We next describe the algorithms studied in this paper. The entire process of DFL can be viewed as a two-stage pipeline: 1) SGD update step, performed locally on each client, and 2) Gossip/Consensus averaging step. We analyze three different scenarios of noise injection, resulting in three different algorithms. 
\subsection{Algorithm 1---FedNDL1}
In this algorithm, each client in parallel performs updates first, see---lines 4--6, and then communicates the updated parameters to their neighbors. The communication depends on the topology of the communication graph, i.e., the mixing matrix, $W$, through a noisy communication channel (line 7). Due to the noisy communication channel, the neighboring client receives a noisy version of the parameters, 
\begin{equation}
    \label{eq:consensus-step}
    x_i^{(t+1)} \ = \ \sum_{j=1}^n w_{ij} \ ( x_j^{(t+\frac{1}{2})} + \delta_j^{(t)} ),
\end{equation}
where $\delta_j^{(t)} \in \mathbb{R}^d$, is a zero mean random noise and  $x_j^{(t+\frac{1}{2})}$ is the vector of parameters sent by client $j$. Since we assume the noise to have a zero mean, the noise variance is
\begin{equation}
    \label{eq:noise-var}
    D^2_{t,j} = \E[\|\delta_j^{(t)}\|^2].
\end{equation}

% \begin{algorithm}[!]
%     \caption{FedNDL1}
%     \label{alg:p1-noisy-DFL}
%     \begin{algorithmic}[1]
%        \STATE {\bfseries Input:} For client $i$,  initialize: $x_i^{(0)} \  \in \ \mathbb{R}^d$, step size $ \{ \eta_t \}_{t=0}^{T-1} $, mixing matrix $W$, noise from the communication channel $\delta^{(t)}$\\
%        \FOR{$t = 0, \dots, T$}
%        \STATE \textbf{Run in parallel for each client $i$}
%        \STATE \ \ \ Sample $\xi_{i}^{(t)} \mbox{, compute } g_i^{(t)} = \widetilde{\nabla} f_i (x_i^{(t)},\xi_i^{(t)})$
%        \STATE \ \ \ $ x_i^{(t+\frac{1}{2})} \ = \  x_i^{(t)} - \eta_t g_i^{(t)} $\
%        \STATE \ \ \ $ x_i^{(t+1)} \ = \ \sum_{j=1}^{n} w_{ij} \ ( x_j^{(t+\frac{1}{2})} + \delta_j^{(t)} ) $\ 
%        \ENDFOR
%     \end{algorithmic}
           %\hspace{.4cm} \implies $ SGD Update
%\implies $ Gossip Averaging Step with channel noise
% \end{algorithm}
\subsection{Algorithm 2---FedNDL2}
Similar to the previous algorithm, this algorithm also performs a two-stage process. However, in this algorithm, we perform the consensus step (line 9) over a noisy communication channel before computing the individual gradients,
\begin{equation}
    \label{eq:p2-gossip}
    x_i^{(t+\frac{1}{2})} =  \sum_{j=1}^{n} w_{ij} ( x_j^{(t)} + \delta_j^{(t)} ),
\end{equation}
where the symbols hold the same meaning as in the previous one. After the gossip averaging step, each client performs the SGD update on their local data (lines 10--12).
% \begin{algorithm}[!]
%     \caption{FedNDL2}
%     \label{alg:p2-noisy-DFL}
%     \begin{algorithmic}[1]
%        \STATE {\bfseries Input:} For each node $i$ initialize: $x_i^{(0)} \  \in \ \mathbb{R}^d$, step size $ \{ \eta_t \}_{t=0}^{T-1} $, mixing matrix $W$, noise from the communication channel $\delta^{(t)}$\\
%        \FOR{$t = 0, \dots, T$}
%        \STATE $ x_i^{(t+\frac{1}{2})} \ = \ \sum_{j=1}^{n} w_{ij} \ ( x_j^{(t)} + \delta_j^{(t)} ) \  $ 
%        \STATE \textbf{Run in parallel for each clients $i$}
%        \STATE \ \ Sample $\xi_{i}^{(t)} \mbox{ ,  } g_i^{(t+\frac{1}{2})}= \widetilde{\nabla} f_i (x_i^{(t+\frac{1}{2})},\xi_i^{(t)})$ \
%        \STATE \ \ $ x_i^{(t+1)} \ = \  x_i^{(t+\frac{1}{2})} - \eta_t g_i^{(t+\frac{1}{2})}  $ 
%        \ENDFOR
%     \end{algorithmic}
% \end{algorithm}
% \subsection{Algorithm 3---FedNDL3}
\subsection{Algorithm 3---FedNDL3}
In FedNDL3, %\Cref{alg:p3-noisy-DFL}, 
the clients share their gradients over a noisy communication channel instead of the weights followed by the SGD update. The reason behind pursuing this idea comes from the motivation for our study of Noisy-FL and the fact that SGD inherently is a noisy process. So, pursuing this scenario gives more flexibility to handle the noise as a part of the SGD process. The entire formulation for this process can be written as, 
\begin{equation}
    \label{eq:p3-gradient-based}
    x_i^{(t+1)} \ = \ x_i^{(t)} - \eta_t \sum_{j=1}^{n} w_{ij} \ ( g_j^{(t)} + \delta_j^{(t)} ),
\end{equation}
where $g_j^{(t)}$ refers to a stochastic gradient of client $j$ at iteration $t$ and the rest of the terms hold the same meaning as before. The algorithms are summarized in the table above.
% \begin{algorithm}[!]
%     \caption{  FedNDL3}
%     \label{alg:p3-noisy-DFL}
%     \begin{algorithmic}[1]
%        \STATE {\bfseries Input:} For each node $i$ initialize: $x_i^{(0)} \  \in \ \mathbb{R}^d$, step size $ \{ \eta_t \}_{t=0}^{T-1} $, mixing matrix $W$, noise from the communication channel $\delta^{(t)}$\\
%        \FOR{$t = 0, \dots, T$}
%        \STATE \textbf{Run in parallel for each client $i$}
%        \STATE \ \ Sample $\xi_{i}^{(t)} \mbox{, compute } g_i^{(t)} := \widetilde{\nabla} f_i (x_i^{(t)},\xi_i^{(t)})$
%        \STATE \ \  $ x_i^{(t+1)} \ = \ x_i^{(t)} - \eta_t \sum_{j=1}^{n} w_{ij} \ ( g_j^{(t)} + \delta_j^{(t)} ) \ $ 
%        \ENDFOR
%     \end{algorithmic}
% \end{algorithm}

\definecolor{light-gray}{gray}{0.85}
\begin{algorithm}[h]
\renewcommand{\thealgorithm}{}
    \caption{\textcolor{darkgray}{FedNDL1},  \textcolor{blue}{FedNDL2}, and \textcolor{purple}{FedNDL3}}
    \label{alg:all-noisy-DFL}
    \begin{algorithmic}[1]
       \STATE {\bfseries Input:} For each node $i$ initialize: $x_i^{(0)} \  \in \ \mathbb{R}^d$, step size $ \{ \eta_t \}_{t=0}^{T-1} $, mixing matrix $W$, noise from the communication channel $\delta^{(t)}$\\
       \FOR{$t = 0, \dots, T$}
       \STATE \textcolor{darkgray}{ \textbf{FedNDL1:} } 
       \STATE \textcolor{darkgray}{ {Run in parallel for each client $i$}}
       \STATE  \textcolor{darkgray}{ Sample $\xi_{i}^{(t)} \mbox{, compute } g_i^{(t)} = \widetilde{\nabla} f_i (x_i^{(t)},\xi_i^{(t)})$}
       \STATE  \textcolor{darkgray}{ $ x_i^{(t+\frac{1}{2})} \ = \  x_i^{(t)} - \eta_t g_i^{(t)} $\ }
       \STATE \textcolor{darkgray}{ $ x_i^{(t+1)} \ = \ \sum_{j=1}^{n} w_{ij} \ ( x_j^{(t+\frac{1}{2})} + \delta_j^{(t)} ) $\ }              
       \STATE \textcolor{blue}{ \textbf{FedNDL2:} } 
      \STATE \textcolor{blue}{ $ x_i^{(t+\frac{1}{2})} \ = \ \sum_{j=1}^{n} w_{ij} \ ( x_j^{(t)} + \delta_j^{(t)} ) \  $ }
       \STATE \textcolor{blue}{ {Run in parallel for each clients $i$}}
       \STATE \textcolor{blue}{ Sample $\xi_{i}^{(t)} \mbox{ ,  } g_i^{(t+\frac{1}{2})}= \widetilde{\nabla} f_i (x_i^{(t+\frac{1}{2})},\xi_i^{(t)})$ \ }
       \STATE \textcolor{blue}{ $ x_i^{(t+1)} \ = \  x_i^{(t+\frac{1}{2})} - \eta_t g_i^{(t+\frac{1}{2})}  $ }

       \STATE \textcolor{purple}{ \textbf{FedNDL3:} }  
       \STATE \textcolor{purple}{ Run in parallel for each client $i$ }
       \STATE \textcolor{purple}{ Sample $\xi_{i}^{(t)} \mbox{, compute } g_i^{(t)} = \widetilde{\nabla} f_i (x_i^{(t)},\xi_i^{(t)})$}
       \STATE \textcolor{purple}{ $ x_i^{(t+1)} \ = \ x_i^{(t)} - \eta_t \sum_{j=1}^{n} w_{ij} \ ( g_j^{(t)} + \delta_j^{(t)} ) \ $ }
       \ENDFOR
    \end{algorithmic}
\end{algorithm}

\subsection{Assumptions}
We now discuss the assumptions we made in our analysis of the proposed algorithms. They are standard assumptions used in the study and analysis of distributed and decentralized algorithms, see~\cite{koloskova2019decentralized,koloskova2020unified,hashemi2021benefits}.
\begin{assumption}[\textbf{Smoothness}]
\label{as1} 
The objective function $\ell(x, \xi)$ is $L$-smooth with respect to $x$, for all $\xi$. Each $f_i(x)$ is $L$-smooth, that is,
\begin{equation}
\|\nabla f_i (x) - \nabla f_i (y)\| \leq L \|x-y\|, \qquad \mbox{for all x,y} .
\end{equation}
Hence the function $f$ is also $L$-smooth.
\end{assumption}
\begin{assumption}[\textbf{Bounded Variance}] 
\label{as2} 
The variance of the stochastic gradient of each client $i$ is bounded, $$\E[||\widetilde{\nabla}f_i (x_i^t,\xi_i^t) - \nabla f_i (x_i^t)||^2] \leq \sigma^2,$$ where $\xi_i^t$ denotes random batch of samples in client node $i$ for $t^{th}$ round, and $\widetilde{\nabla}f_i(x_i^t,\xi_i^t)$ denotes the stochastic gradient. In addition,  we also assume that the stochastic gradient is unbiased, i.e., $\E[\widetilde{\nabla}(f_i (x_i^t,\xi_i^t))] = \nabla f_i (x_i^t)$.
\end{assumption}
\begin{assumption}[\textbf{Mixing matrix}] 
\label{as3} 
The mixing matrix W satisfies for $\rho \in (0,1]$,
\begin{equation*}
    \|(\Bar{X} - X)W \|_{F}^2 \leq (1 - \rho)\|\Bar{X} - X\|_{F}^2,
\end{equation*}
which means that the gossip averaging step brings the columns of $X \in \mathbb{R}^{d \times n}$ closer to the row-wise average, that is, $\Bar{X}= X\frac{\mathds{1}\mathds{1}^{\top}}{n}$.
\end{assumption}
\begin{assumption}[\textbf{Bounded Client Dissimilarity (BCD)}] 
\label{as4} 
For all $x \in \mathbb{R}^{d}$,
\begin{equation*}
    \label{eq:bcd}
    \frac{1}{n}\sum_{i=1}^{n}\|\nabla f_{i}(x) - \nabla f(x)\|^2 \leq B^2,
\end{equation*}
where B is a constant. 
\end{assumption}
The above assumption is made to limit the extent of client heterogeneity and is standard in the DFL setup. While methods based on gradient tracking \cite{nedic2017achieving} do no require this assumption, they suffer from increased communication cost and the variance, which limits their practicality \cite{yuan2020can}. Note that this assumption is only used in the analysis of FedNDL1 and FedNDL2.
\begin{assumption}[\textbf{Noise model}] 
\label{as5} 
The noise present due to contamination of communication channel $\delta_i^{(t)}$ is independent, has zero mean and bounded variance, that is, $\E[\delta_i^{(t)}] = 0$ and $\E[||\delta_i^{(t)}||^2] = D_{t,i}^2 < \infty$. 
\end{assumption}
This assumption is specific to the imperfect information sharing setup and is considered recently in \cite{reisizadeh2022distributed,reisizadeh2022almost,wei2022federated}.
\begin{assumption}[\textbf{Bounded Recursive Consensus Error}] 
\label{as6} 
Let the consensus error be defined as $(C.E)_t=\frac{1}{n}\|\Bar{X}_{t} - {X}_{t}\|^2_F$. We assume that the consensus error is upper bounded,
\begin{equation*}
    \E[(C.E)_{t+1}] \leq \rho_t \E[(C.E)_{t}] + \gamma_{t},
\end{equation*} 
where $\rho_t \in (0,1)$ and $\gamma_t \geq 0$.   
\end{assumption}
\begin{remark}
We use this assumption in the convergence analysis of FedNDL3. %\Cref{alg:p3-noisy-DFL}. 
Theoretically speaking the above assumption \ref{as6} can be viewed as a general formulation one obtains while unfolding the recursion of consensus error in the analysis of decentralized SGD---see, e.g., \cite{koloskova2020unified}. Additionally, this assumption is satisfied by FedNDL1 and FedNDL2 with $\rho_t = \rho$ and certain $\{\gamma_t\}$ that depends on other parameters of the problem. We further note that using multi-round gossiping \cite{hashemi2021benefits} or acceleration methods such as  Chebyshev acceleration \cite{arioli2014chebyshev,scaman2017optimal} this assumption may be satisfied by FedNDL3 as well and  $\rho_t$ and $\gamma_t$ can be significantly reduced. \end{remark}
% $
% X_{t}^{i} : \text{Parameter/Weights of client $i$ at time $t$}
% \\
% \nabla_{t,i} = \nabla f(X_{t}^{i}) \text{  and  } \widetilde \nabla_{t,i} = \widetilde \nabla f(X_{t}^{i}, \xi_{t}^{i})\\
% \Bar{X}_t = \frac{1}{n}\sum_{i=1}^{n}X_{t}^{i}\\
% $
% The consensus error can be defined as, \\
% Error = $\frac{1}{n}\sum_{i=1}^{n}\E[\|X_{t,i} - \Bar{X}_{t}\|^2]$
\section{Convergence Analysis}
%\section{Some Results of the Convergence Analysis of the Proposed Algorithms}
In this section, we state the main theorem providing an upper bound on the convergence error of FedNDL1, FedNDL2 and FedNDL3.
% Due to time constraints, we cannot analyze \textbf{Algorithm 2} and \textbf{Algorithm 3}.  
The convergence results for all the algorithms are for a non-convex $L$-smooth loss function. These results are for the case when there is noise present due to the imperfection of the communication channel, % and also it does not include the consensus error term.

\begin{theorem}[\textbf{Smooth non-convex cases for Noisy-DFL}]
\label{thm-all}
Suppose Assumptions \ref{as1}, \ref{as2}, \ref{as3}, \ref{as4}, \ref{as5} and \ref{as6} (only for FedNDL3) hold. Let $\eta L < \frac{1}{12}$, $\eta = \mathcal{O}(\frac{1}{\sqrt{T}})$ and $\Bar{D}^2= \frac{1}{nT} \sum_{t,i = 1,1}^{T,n}D^2_{t,i}$, then we have:
\begin{itemize}
    \item \textbf{FedNDL1}: For $\eta L < \frac{\rho}{2\sqrt{6}}$,
    \begin{multline}
        \label{eq:thm1-FedNDL1}
        \frac{1}{T}\sum_{t=1}^{T}\E[\|\nabla f(\Bar{x}_t)\|^2] + \frac{L^2}{T}\sum_{t=1}^{T}\E[(C.E)_t] = \mathcal{O}\Big(\frac{\rho}{n\sqrt{T}} \sigma^2 + \frac{\rho^2}{T} B^2
     + \frac{T^{\frac{3}{2}}}{\rho} \Bar{D}^2\Big),
    \end{multline}
    % $\frac{1}{T}\sum_{t=1}^{T}\E[\|\nabla f(\Bar{X}_t)\|^2] + \frac{L^2}{T}\sum_{t=1}^{T}\E[(C.E)_t] \\= \mathcal{O}\Big(\frac{\rho}{n\sqrt{T}} \sigma^2 + \frac{\rho^2}{T} B^2
    %  + \frac{\sqrt{T}}{\rho n} \sum_{t,i = 1,1}^{T,n}D^2_{t,i}\Big)$
\end{itemize}
\begin{itemize}
    \item \textbf{FedNDL2}: For $\eta L < \frac{\rho}{4\sqrt{3}}$,
    \begin{multline}
        \label{eq:thm2-FedNDL2}
        \frac{1}{T}\sum_{t=1}^{T}\E[\|\nabla f(\Bar{x}_t)\|^2] + \frac{L^2}{T}\sum_{t=1}^{T}\E[(C.E)_t] = \mathcal{O}\Big(\frac{\rho}{n\sqrt{T}} \sigma^2 + \frac{\rho^2}{T} B^2
     + \frac{T^{\frac{3}{2}}}{\rho} \Bar{D}^2\Big),
    \end{multline}
    % \\$\frac{1}{T}\sum_{t=1}^{T}\E[\|\nabla f(\Bar{X}_t)\|^2] + \frac{L^2}{T}\sum_{t=1}^{T}\E[(C.E)_t] \\= \mathcal{O}\Big(\frac{\rho}{\sqrt{T}} \sigma^2 + \frac{\rho^2}{T} B^2
    %  + \frac{\sqrt{T}}{\rho n} \sum_{t,i = 1,1}^{T,n}D^2_{t,i}\Big)$
\end{itemize}
\begin{itemize}
    \item \textbf{FedNDL3}:
    \begin{multline}
        \label{eq:thm3-FedNDL3}
        \frac{1}{T}\sum_{t=1}^{T}\E[\|\nabla f(\Bar{x}_t)\|^2] + \frac{L^2}{T}\sum_{t=1}^{T}\E[(C.E)_t] = \mathcal{O}\Big(\frac{1}{n\sqrt{T}} \sigma^2 + \frac{1}{T}\sum_{t=1}^{T}\frac{\gamma_t}{\rho_t}
     + \frac{1}{\sqrt{T}} \Bar{D}^2\Big),
    \end{multline}
    % \\$\frac{1}{T}\sum_{t=1}^{T}\E[\|\nabla f(\Bar{X}_t)\|^2] + \frac{L^2}{T}\sum_{t=1}^{T}\E[(C.E)_t] \\= \mathcal{O}\Big(\frac{1}{n\sqrt{T}} \sigma^2 + \frac{1}{T}\sum_{t=1}^{T}\frac{\gamma_t}{\rho_t}
    %  + \frac{1}{T^{\frac{3}{2}}} \sum_{t,i = 1,1}^{T,n}D^2_{t,i}\Big)$
\end{itemize}
where all expectations are w.r.t. the data and the noise.
\end{theorem}
\Cref{thm-all} establishes a worst-case upper bound on the convergence of the three algorithms studied in the paper. In particular, the theorem jointly bounds the expected gradient norm (which is a notion of approximate first-order stationarity of the average iterate $\bar{x}_t$) and the consensus error. The convergence bounds consist of three terms: the first term effectively captures the error arising from inaccurate initialization and stochasticity of the first-order oracle, which matches the error of centralized SGD. The second term captures the effect of data heterogeneity, and the last term captures the adverse effect of imperfect communication modeled as communication noise. 

In addition, \Cref{thm-all} captures the impact of presence channel noise on the convergence of studied algorithms. Specifically, \eqref{eq:thm1-FedNDL1} and \eqref{eq:thm2-FedNDL2} indicate that FedNDL1 and FedNDL2 suffer from a severe impact of noise on the worst-case convergence: as the number of communication rounds/iterations $T$ increases the guarantee on finding a stationary solution and consensus error weakens. In fact, the error increases with $T$. We will verify these results further numerically in \Cref{sec:exp}. Furthermore, as the connectivity of the communication graphs decreases (which corresponds to a smaller $\rho$), the impact of noise increases. This point is further verified in \Cref{sec:exp}. With regard to FedNDL3, however, \Cref{thm-all} establishes that the algorithm is resilient to the impact of noise. In particular, in contrast with the convergence bound of FedNDL1 and FedNDL2, the last term in \eqref{eq:thm3-FedNDL3} decreases with $T$. Intuitively, this theoretically-grounded property is linked to SGD which inherently is a noisy process and thus is more resilient towards added noise to a certain degree. \Cref{thm-all} further shows that, different from FedNDL1 and FedNDL2, the impact of noise is independent of the communication topology as the last term in \eqref{eq:thm3-FedNDL3} is independent of $\rho_t$ and $\gamma_t$. In \Cref{sec:exp} we numerically verify these two distinguishing properties of FedNDL3.

 % In the theorem above the expectation is w.r.t. the data and the noise. As mentioned in the previous section the motivation behind analysing FedNDL3 comes from the analysis of Noisy-FL which shows the impact of downlink (server to client) noise is more detrimental in comparison to the uplink (client to server) noise. Intuitively this can also be observed from SGD which inherently is a noisy process and thus is more resilient towards added noise to a certain degree. From \Cref{thm-all}, \Cref{eq:thm3-FedNDL3}, we can see that the noise term is of order $\mathcal{O}({T^{-\frac{1}{2}}})$.
 \subsection{Proof-Sketch}
 % Due to the limitation on the number of pages, we now provide a proof sketch for \Cref{thm-all}. The detailed proofs for the algorithms are available at \url{https://abolfazlh.github.io/files/fedndl.pdf}. W
 e start the proof by upper bounding the second moment of the gradient on the average of iterates by using the $L$-smoothness of the loss function. This step is a standard practice used in the convergence proofs for non-convex, $L$-smooth loss functions. The second moment here is bounded by the inaccurate initialization, the variance of the stochastic gradients, noise present due to imperfect channels, and the consensus error function, $(C.E)_{t}$. For FedNDL3, in the next step, we proceed with upper bounding the $(C.E)_{t+1}$ followed by defining a potential function,
 \begin{equation}
     \label{eq:DFL_potential}
     \psi_{t} = \E[f(\Bar{x}_{t})] + \phi (C.E)_{t}, {\text{  where $\phi>0$ is a constant.}}
 \end{equation}
 This new potential function enables us to jointly bound the expected gradient norm and the consensus error without requiring restrictive and impractical assumptions such as the bounded gradient norm assumption.
 Another noteworthy point here is the term $\phi$ for FedNDL3 which now, unlike in \eqref{eq:DFL_potential}, is time-dependent and is denoted as $\phi_t$. Now, we telescope over the $\psi_{t+1} - \psi_{t}$ for $t=\{1,\ldots, T\}$ and dividing it by $T$, along with specific choices of $\phi$ and $\phi_t$ results in \Cref{thm-all}. 
 Detailed proofs for FedNDL1, FedNDL2 and FedNDL3 are provided in the 
%  \textcolor{red}{abolfazl: perhaps we can here instead cite a link to the proofs. Also, please add a short subsection highlighting the key ideas in the proofs.} 
% In the following section, we present the numerical results for the three proposed algorithms.

\captionsetup{font=normal}
\captionsetup[sub]{font=normal}
\begin{figure*}[!]
\vspace{0.1in}
% \centering
\begin{subfigure}{0.33\textwidth}
    \centering    \includegraphics[width=\textwidth]{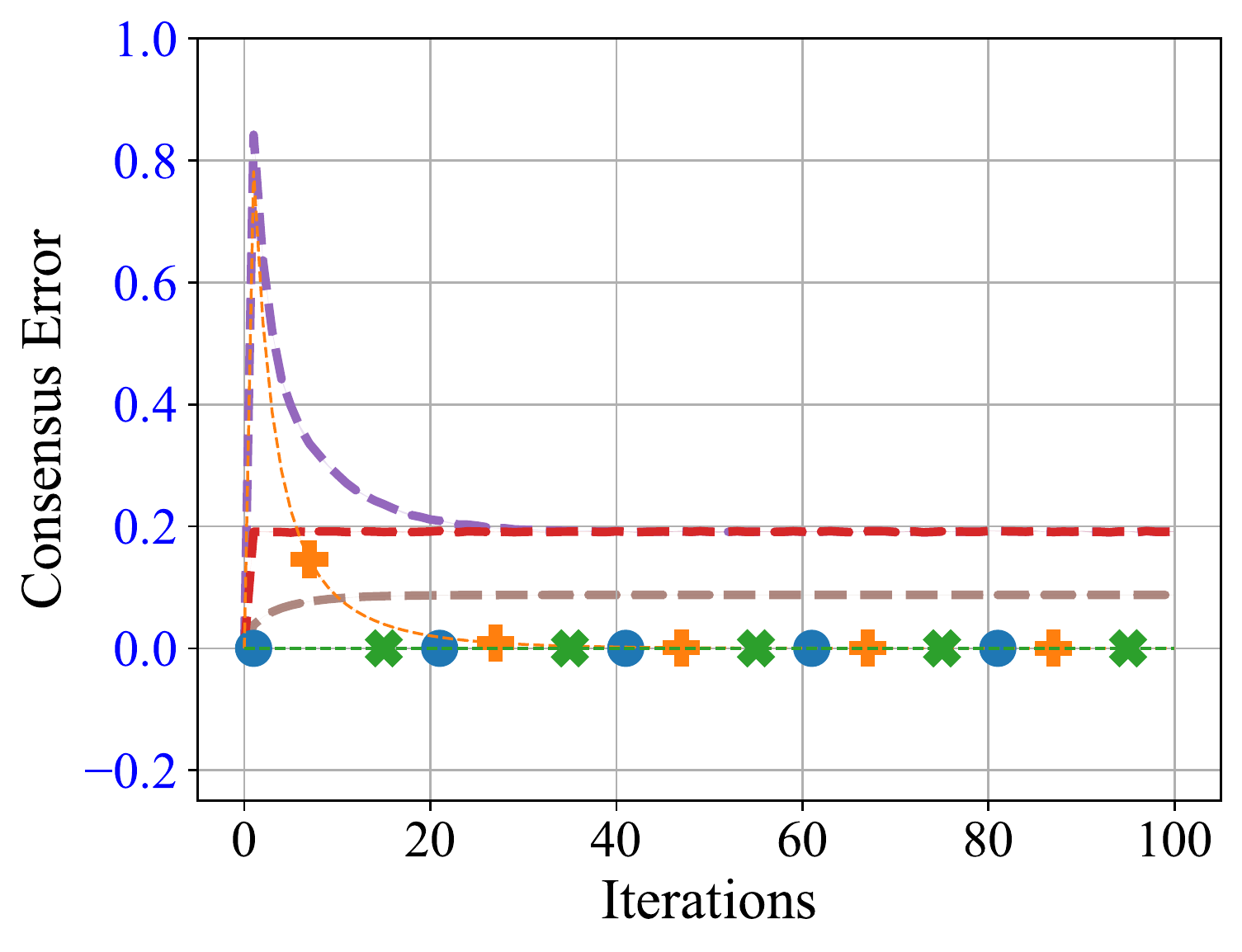}
    \caption{Fully-connected topology}
    \label{fig:syn_1_0}
\end{subfigure}
\begin{subfigure}{0.33\textwidth}
    \centering
    \includegraphics[width=\textwidth]{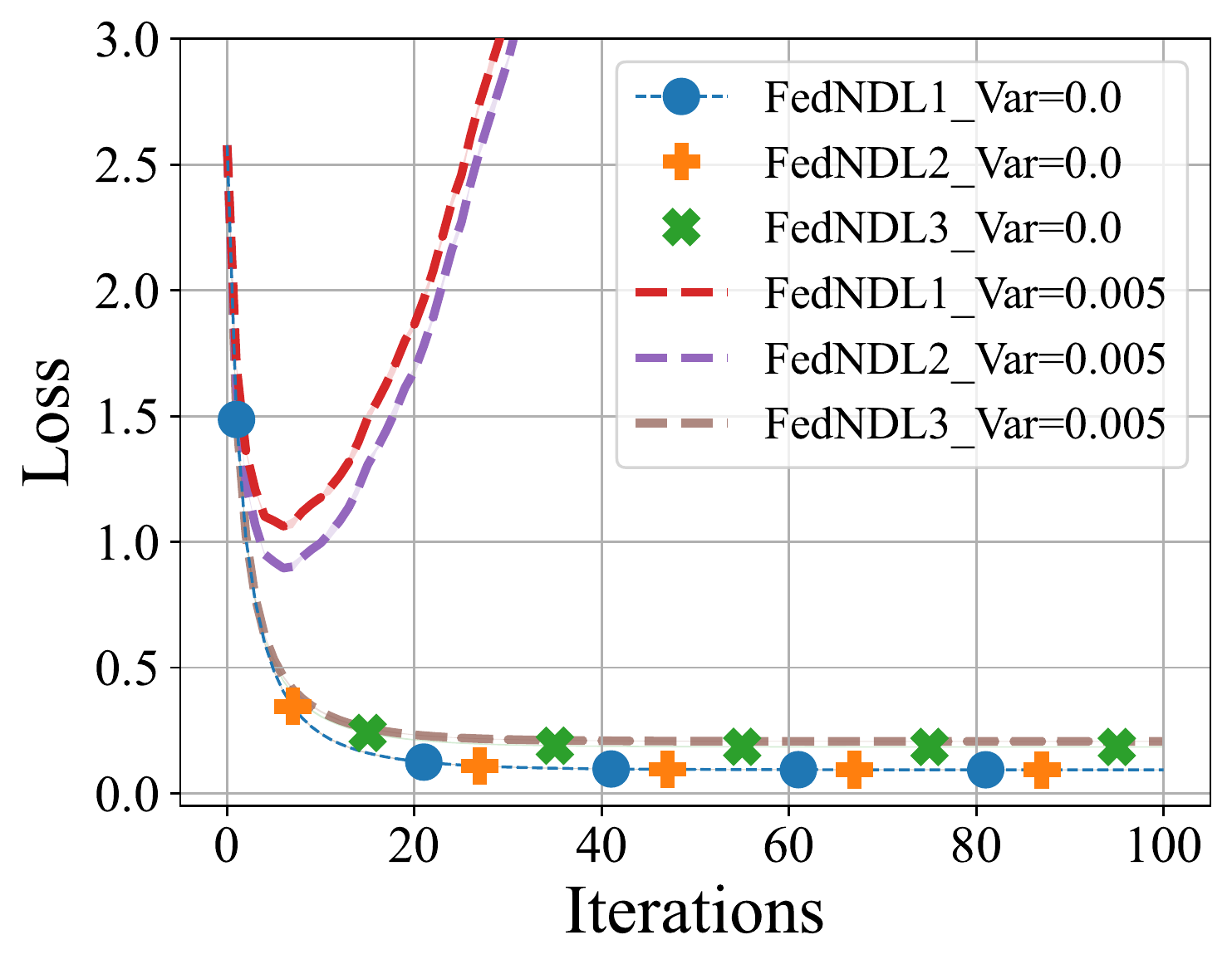}
    \caption{Torus topology}
    \label{fig:syn_1_0.005}
\end{subfigure}
\begin{subfigure}{0.33\textwidth}
    \centering
    \includegraphics[width=\textwidth]{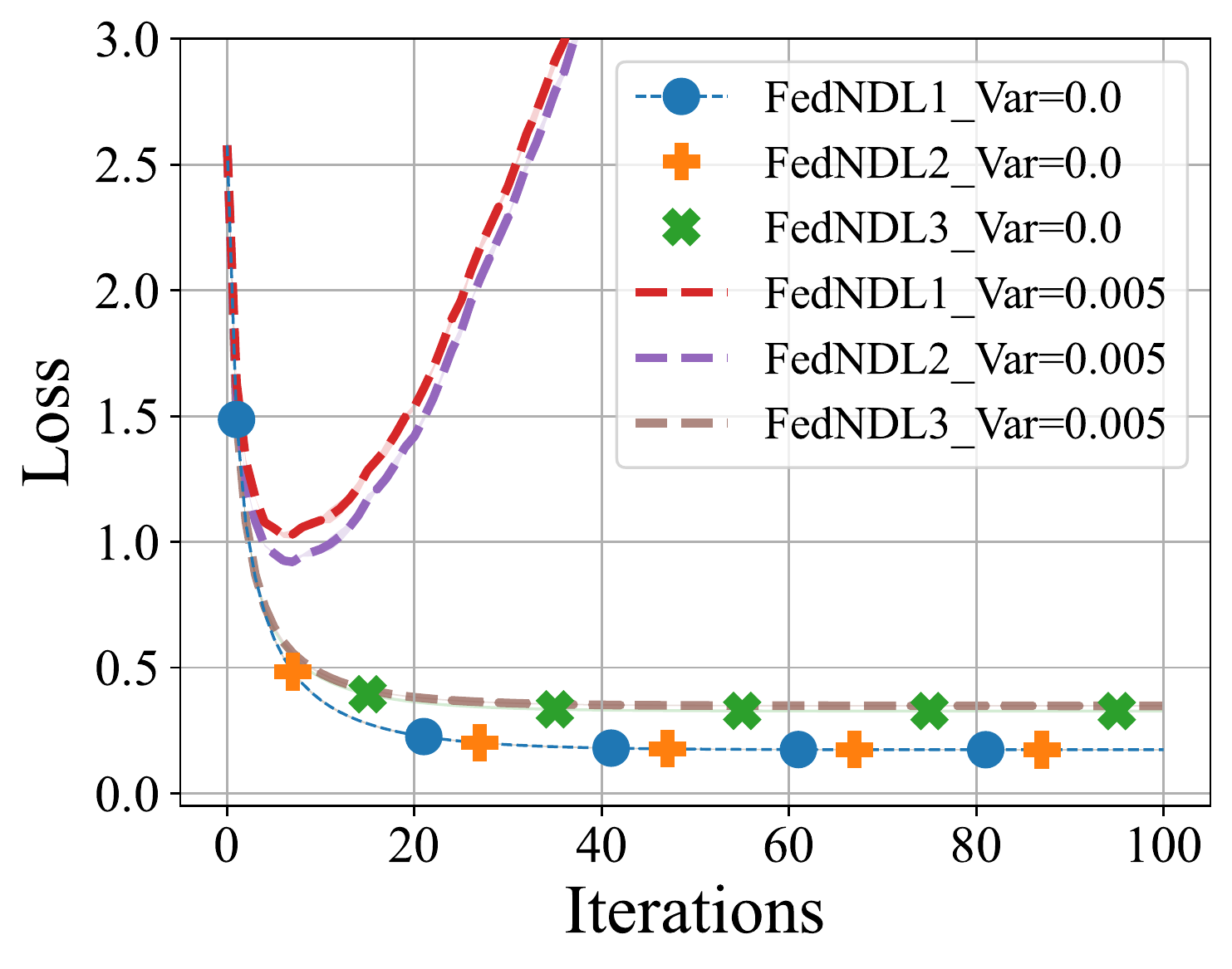}
    \caption{Ring topology} 
    \label{fig:syn_1_0.001}
\end{subfigure} 
\caption{Loss vs. iterations with and without noise for different communication topologies.}
\label{fig:loss_all}
\vspace{-4mm}
\end{figure*}

\begin{figure*}[!]
% \centering
\begin{subfigure}{0.33\textwidth}
    \centering
    \includegraphics[width=\textwidth]{Figures/Syn1/syn1_fully_connected_var0.005_plotconsensus.pdf}
    \caption{Fully-connected topology} 
    \label{fig:syn_1_0.005_full_cons}
\end{subfigure}
\begin{subfigure}{0.33\textwidth}
    \centering
    \includegraphics[width=\textwidth]{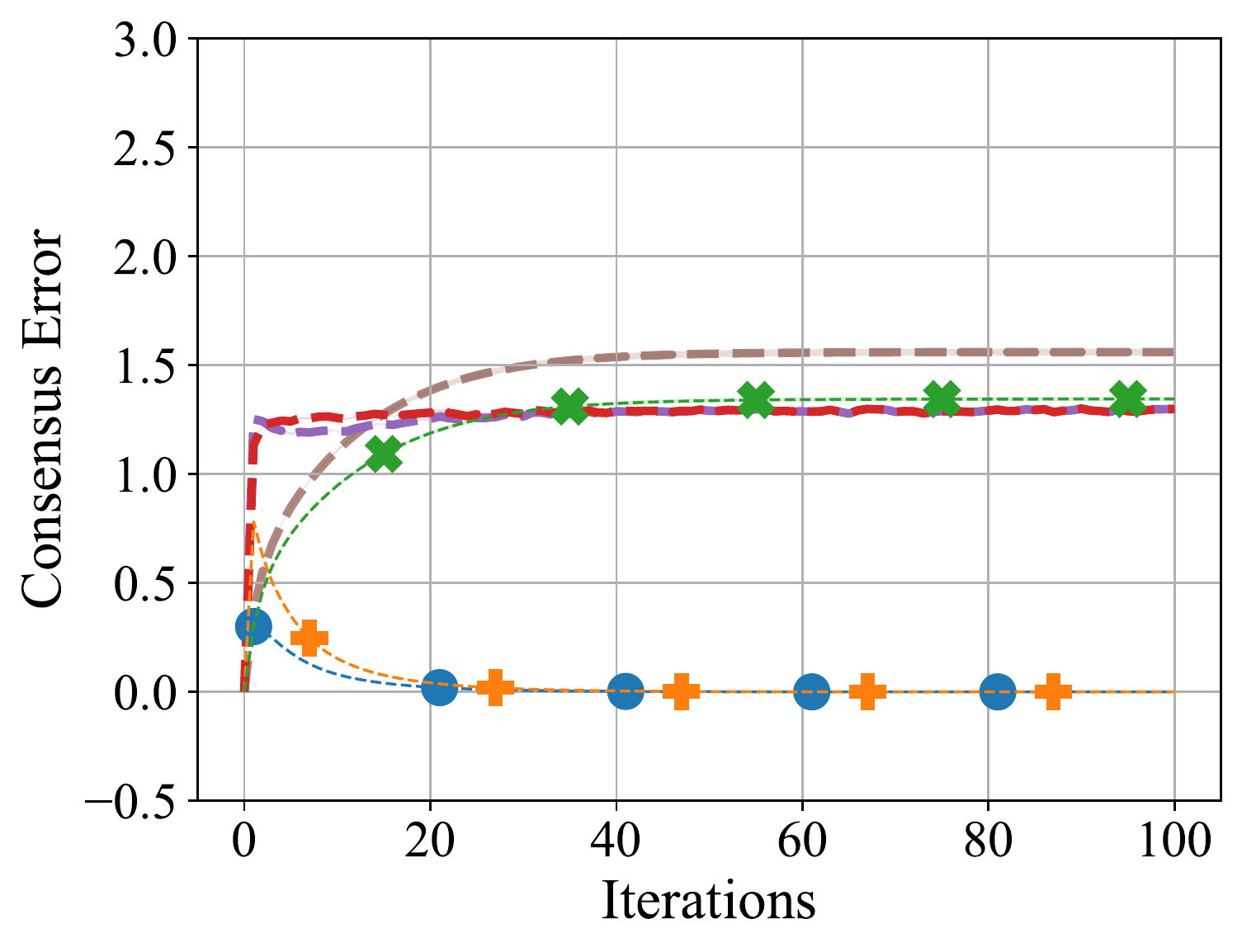}
    \caption{Torus topology}
    \label{fig:syn_1_0.005_torus_cons}
\end{subfigure}
\begin{subfigure}{0.33\textwidth}
    \centering
    \includegraphics[width=\textwidth]{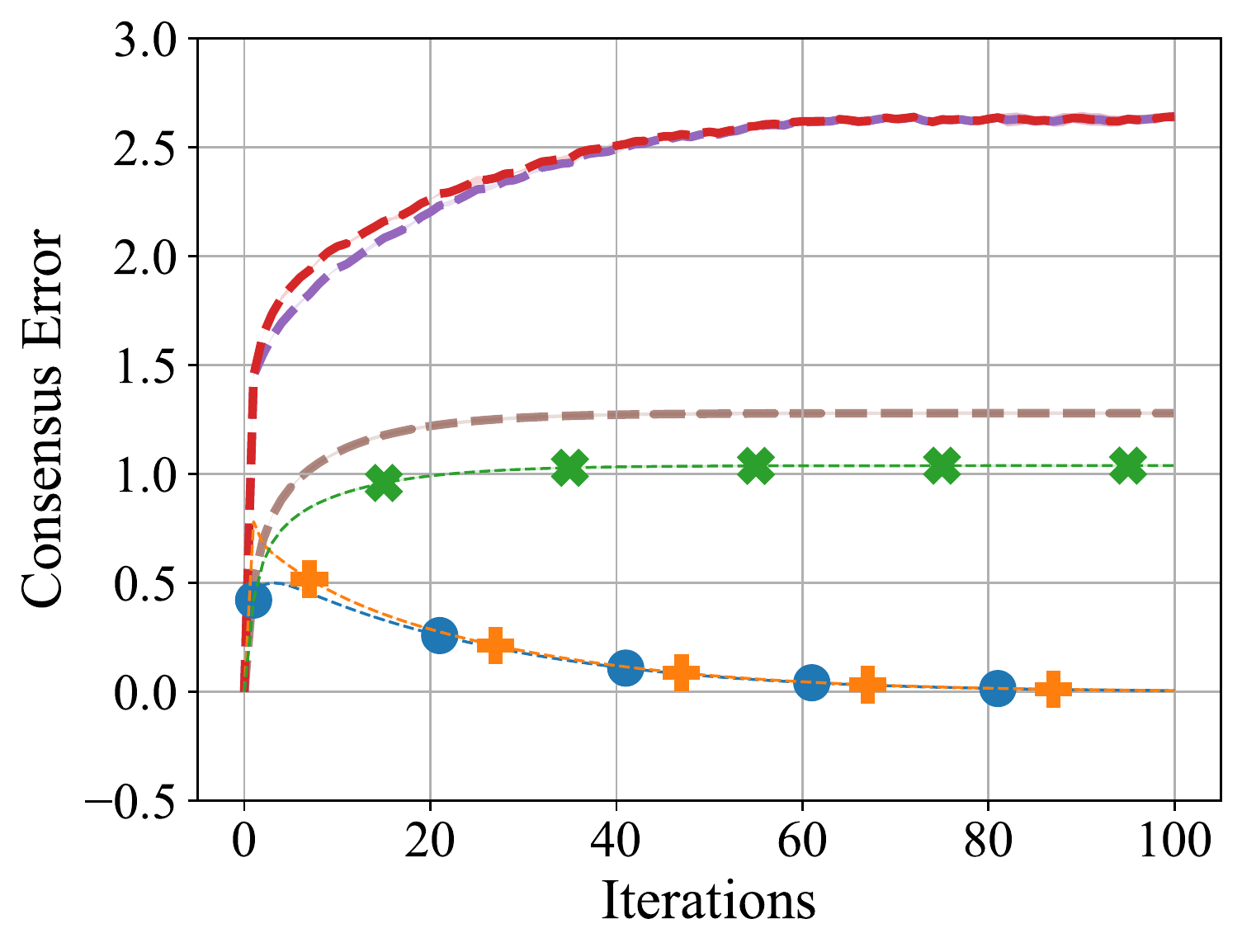}
    \caption{Ring topology} 
    \label{fig:syn_1_0.005_ring_cons}
\end{subfigure}
\caption{Consensus error vs. iterations with and without noise for different communication topologies. Note the different Y-axis scale in Fig (a) as compared with Fig (b) and (c) for better readability.}
\label{fig:consensus_all}
\vspace{-4mm}
\end{figure*}

\section{Experiments}\label{sec:exp}
In this section, we perform several experiments on regression problems to verify the impact of noise on the convergence of the three proposed algorithms as established in~\Cref{thm-all}. 
% We perform experiments on two widely used network topologies namely ring and torus. 
We consider the case when the number of clients, $n$ = 16. The experiments are repeated three times, and the results (loss/consensus error) are averaged. We use the mean-squared error loss function with $L_2$ regularization. The learning rate of the model is set as 0.2 with a decay of 0.9 with every iteration. We describe the dataset generation next. 
%Describe mixing matrix - Add couple of sentences

% The regression experiments are run for 100 iterations while the deep learning experiments are run for 300 epochs. 
 
% \textit{Logistic Regression:} We generate a custom dataset from  MNIST \cite{lecun1998gradient} where we sorted the data as per the labels (0 to 9). The data is then divided into two categories. The first five classes (0 to 4) are assigned as class 0 and the remaining classes (5-9) are assigned as class 1. Furthermore, we use binary cross-entropy loss for the classification problem. The data is distributed equally among all 16 clients where each client has data from at most both classes making this a strictly heterogeneous dataset.
\textit{Linear Regression Dataset:} We generate a synthetic data samples ($m$ = 10000)  ${(x_i; y_i)}^{m}_{i=1}$ according to $  y_i = \langle w, x_i \rangle  + \epsilon_i $, where $ w \in \mathbb{R}^{2000}$, %the $i^{th}$ input%
$x_i \sim \mathcal{N}(0; I_{2000})$ and noise, $\epsilon_i \sim \mathcal{N}(0, 0.05)$. %. In the experiments that follow, we refer to this dataset as SYN-1.   

% \textit{SYN-2 Dataset - Non-Linear Regression:}  We generate a synthetic data samples ($m$ = 10000)  ${(x_i; y_i)}^{m}_{i=1}$ generated according to $  y_i = \text{relu} ( \langle w , x_i \rangle ) + \epsilon_i $, where $ w \in \mathbb{R}^{2000}$, $x_i \sim \mathcal{N}(0; I_{2000})$ and  noise, $e_i \sim N(0, 0.05)$ 
%Synthetic data samples ($m$ = 10000) ${(x_i; y_i)}^{m}_{i=1}$ generated according to $  y_i = \text{relu} ( \langle w , x_i \rangle ) + \epsilon_i $, where $ w \in \mathbb{R}^{2000}$, the $i^{th}$ input $x_i \sim \mathcal{N}(0; I_{2000})$,  noise $e_i \sim N(0, 0.05)$ and . ReLu (relu(z) = max(z,0)) is a non-linear function and thereby formulates the regression problem as non-linear. %We use the mean-squared error loss function with $L_2$ regularization for the non-linear regression problem. In the experiments that follow, we refer to this dataset as SYN-2.

% \textit{Deep Learning Classification: } We perform large-scale non-convex optimization tasks on MNIST and Fashion MNIST datasets. We consider 16 clients with torus and ring topologies. The model in each client is a two-layer neural network with a non-linear ReLU activation function with 300 neurons in each hidden layer. We perform stochastic gradients with a batch size of 256. The data is divided equally and assigned to each client. The performance of the decentralized global model is evaluated on the test data. Analysis of the impact of the noise on the heterogeneous datasets is part of future studies. 

The experiments are performed with various levels of noise variance, $D_{t,i}^2$ for all $t,i$, as described in the algorithms, for various communication topologies, namely the ring, torus, and fully connected network. The mixing matrix can be defined as a weighted adjacency matrix of a given communication graph\cite{xiao2004fast}. The nonzero weights in the mixing matrix for ring topology are equal to $\frac{1}{3}$, in the torus topology $\frac{1}{5}$, and fully connected topology, $\frac{1}{n}$. 

We first perform the experiments with no noise as a baseline and then gradually increase the noise variance to study the robustness of the algorithms. For the purpose of consistency, we have shown the results of the experiments with noise variance  $D_{t,i}^2 = 0.005$ in \Cref{fig:loss_all,fig:consensus_all} along with no noise scenario. The experiments were performed on Intel Xeon Gold workstation.
\subsection{Discussion}
We performed numerical experiments on the algorithms with and without noise for different topologies. We observed, see~\Cref{fig:loss_all}, that the algorithms FedNDL1 and FedNDL2 perform poorly in terms of convergence due to noise presence which is consistent with \Cref{thm-all}. In addition, as it can be seen in~\Cref{fig:consensus_all}, the consensus error also increases with the noise which is also consistent with our theoretical analysis presented in~\Cref{thm-all}. 

On the other hand, the algorithm FedNDL3 is observed to be the most robust as it does not diverge in the presence of added channel noise. The noise term for the FedNDL3 in the upper bound given in~\Cref{thm-all} is of order $\mathcal{O}({T^{-\frac{1}{2}}})$, whereas it is of order $\mathcal{O}({T^{\frac{3}{2}}})$ for FedNDL1 and FedNDL2.  This effect of the noise can also be observed in ~\Cref{fig:loss_all}.
% The performance of FedNDL1 and FedNDL2 are similar algorithms and start diverging with noise injection. 

The consensus error is strongly dependent on the communication network. We can observe that the consensus error is much lower for the fully connected network and significantly higher for the ring topology for the same algorithm in the presence of noise. The consensus error function plots in~\Cref{fig:consensus_all} are consistent with the connectivity of the communication network (number of client interactions). The fully connected topology encompasses the maximum number of clients interaction. It therefore, yields the lowest consensus error, followed by the torus topology and then the ring topology, which has the lowest number of client interactions.

% \textcolor{red}{abolfazl: see the discussion following theorem 1 and discuss elaborately the impact of topology, noise, and the number of iterations. Make sure to refer to the specific terms, eq. in theorem 1.}

\section{Conclusion and Future Work}
We studied the impact of noisy communication channels on the convergence of DFL. We proposed multiple scenarios for establishing consensus in the presence of noise and provided experimental results on all the algorithms. Additionally, we provided theoretical results for FedNDL1, FedNDL2 and FedNDL3, under the assumption of smooth non-convex function, and we observed that in FeDNDL3, the noise term in the upper bound given by~\Cref{thm-all} is of order $\mathcal{O}({T^{-\frac{1}{2}}})$ and is independent of communication topology. In contrast, the impact of noise on the convergence of FedNDL1, FedNDL2 increases with $T$ and weaker communication structures. We conducted numerical experiments on a synthetically generated dataset and observed that FedNDL3 is more robust against the added noise than the other two algorithms analyzed in this paper. 

Future research should focus on a formal study of the benefits of multi-round gossiping \cite{hashemi2021benefits}, or acceleration methods such as Chebyshev acceleration \cite{arioli2014chebyshev,scaman2017optimal}, and establishing statistical lower bounds on the convergence of the algorithms under the presence of noise in communication channels. The algorithms should also be tested on large-scale deep-learning datasets as well as for different network topologies.
%%%%%%%%%%%%%%%%%%%%%%%%%%%%%%%%%%%%%%%%%%%%%%%%%%%%%%%%%%%%%%%%%%%%%%%%%%%%%%%%%
% \section{ACKNOWLEDGMENTS}

% \textcolor{red}{abolfazl: there is a weird, added spacing between the references which shouldn't be there (compared to my previous ACC and CDC papers). Likely, it is because of a package you added. Or may be they have updated ieeeconf.cls? ANS: it was because of adding parskip! why do we need this package? we also had "thispagestyle{empty}" "pagestyle{plain}" in the beginning which messed up the formatting!}
% %%%%%%%%%%%%%%%%%%%%%%%%%%%%%%%%%%%%%%%%%%%%%%%%%%%%%%%%%%%%%%%%%%%%%%%%%%%%%%%%
%\newpage
\bibliographystyle{ieeetr}
\bibliography{refs.bib}
\clearpage
\appendices
% \section{Proof of the Algorithm - FedNDL1}
\section{Theorems}
% \begin{theorem}[\textbf{Smooth non-convex cases for Noisy-DFL}]
% Suppose Assumptions \ref{as1}, \ref{as2}, \ref{as3}, \ref{as4}, \ref{as5} and \ref{as6}(only for FedNDL3) hold. Let $\eta L < \frac{1}{12}$ and $\eta = \mathcal{O}(\frac{1}{\sqrt{T}})$, we have
% \begin{itemize}
%     \item \textbf{FedNDL1}: $\frac{1}{T}\sum_{t=1}^{T}\E[\|\nabla f(\Bar{X}_t)\|^2] + \frac{L^2}{T}\sum_{t=1}^{T}\E[(C.E)_t] = \mathcal{O}\Big(\frac{1}{\sqrt{T}} \sigma^2 + \frac{1}{T^{\frac{3}{2}}} B^2
%      + \frac{\sqrt{T}}{n} \sum_{t,i = 1,1}^{T,n}D^2_{t,i}\Big)$
% \end{itemize}
% \begin{itemize}
%     \item \textbf{FedNDL2}: $\frac{1}{T}\sum_{t=1}^{T}\E[\|\nabla f(\Bar{X}_t)\|^2] + \frac{L^2}{T}\sum_{t=1}^{T}\E[(C.E)_t] = \mathcal{O}\Big(\frac{1}{\sqrt{T}} \sigma^2 + \frac{1}{T^{\frac{3}{2}}} B^2
%      + \frac{\sqrt{T}}{n} \sum_{t,i = 1,1}^{T,n}D^2_{t,i}\Big)$
% \end{itemize}
% \begin{itemize}
%     \item \textbf{FedNDL3}: $\frac{1}{T}\sum_{t=1}^{T}\E[\|\nabla f(\Bar{X}_t)\|^2] + \frac{L^2}{T}\sum_{t=1}^{T}\E[(C.E)_t] = \mathcal{O}\Big(\frac{1}{n\sqrt{T}} \sigma^2 + \frac{1}{T}\sum_{t=1}^{T}\frac{\gamma_t}{\rho_t}
%      + \frac{1}{T^{\frac{3}{2}}} \sum_{t,i = 1,1}^{T,n}D^2_{t,i}\Big)$
% \end{itemize}
% \end{theorem}
\begin{theorem}[\textbf{Smooth non-convex case for \textbf{FedNDL1}}] 
    \label{thm-Algorithm-1}
    Suppose Assumptions \ref{as1}, \ref{as2}, \ref{as3}, \ref{as4} and \ref{as5} hold. Let $\eta L < \frac{1}{6}$, $\rho > 2\sqrt{6} \eta L$, and $f^{*} = \min_{x \in \mathbb{R}^{d}} f(x)$, for $\phi > \frac{L^2 \eta(\frac{1}{2}+2L \eta)}{[\frac{\rho}{2} - \frac{12 L^2 \eta^2}{\rho}]}$, we have , 
    \begin{multline}
        \frac{1}{T}\sum_{t=1}^{T}\E[\|\nabla f(\Bar{X}_t)\|^2] + 
        \frac{L^2}{2}\frac{1}{T}\sum_{t=1}^T\E[(C.E)_t]
        \\
        \leq \frac{2(f(\Bar{X}_1) - f^{*} + \phi(C.E)_{1})}{\eta(1 - 4 \eta L)} + \frac{2\eta(\phi(1-\rho)+\frac{L}{n})}{(1 - 4 \eta L)}\sigma^2 
        \\
        + \frac{16 \phi\eta B^2}{(1 - 4 \eta L)\rho} 
     + \frac{2\Big(L+\phi\big[\frac{\rho}{2}+\frac{2}{\rho}\big]\Big)}{\eta(1 - 4 \eta L)}\frac{1}{nT}\sum_{t=1}^{T}\sum_{i=1}^{n}{D}^2_{t,i}
    \end{multline}
\end{theorem}
\begin{theorem}[\textbf{Smooth non-convex case for \textbf{FedNDL2}}] 
    \label{thm-Algorithm-2}
    Suppose Assumptions \ref{as1}, \ref{as2}, \ref{as3}, \ref{as4} and \ref{as5} hold. Let $\eta L < \frac{1}{12}$, $\rho > 4\sqrt{3} \eta L$ and $f^{*} = \min_{x \in \mathbb{R}^{d}} f(x)$, for $\phi > \frac{L^2 \eta(1+2L \eta)(1-\rho)}{[\frac{\rho}{2} - \frac{24 L^2 \eta^2}{\rho}]}$, we have 
    \begin{multline}
        \frac{1}{T}\sum_{t=1}^T\E[\|\nabla f(\Bar{X}_t)\|^2] +
         \frac{L^2}{T}\sum_{t=1}^T\E[(C.E)_t] 
        \\
        \leq \frac{2(f(\Bar{X}_1) -f^{*} + \phi(C.E)_{1})}{\eta(1 - 2\eta L)T} + \frac{2\eta(\phi+\frac{L}{n})}{(1 - 2\eta L)}\sigma^2 
        + \frac{8\phi\eta(1+\frac{2}{\rho})}{(1 - 2\eta L)}B^2
        \\
        + \frac{2\Big(L^2[\eta(1+2 L \eta) + 2] + 4\phi[1-\rho + \frac{24 L^2 \eta^2}{\rho}]\Big)}{\eta(1 - 2\eta L)}\frac{1}{nT}\sum_{t=1}^T \sum_{i=1}^n D_{t,i}^2 
    \end{multline}
\end{theorem}
\begin{theorem}[\textbf{Smooth non-convex case for FedNDL-3}] 
    \label{thm:p3-noisy-DFL}
    Suppose Assumptions \ref{as1}, \ref{as2}, \ref{as3}, \ref{as4}, \ref{as5} and \ref{as6} hold. Let $\eta L < \frac{1}{6}$ and $f^{*} = \min_{x \in \mathbb{R}^{d}} f(x)$, for $\phi_t = L^2 \eta (1+2L\eta) + 2\phi_{t+1}\rho_t$, we have
    \begin{multline}
        \frac{1}{T}\sum_{t=1}^{T}\E[\|\nabla f(\Bar{X}_t)\|^2] + \frac{2L^2}{T}\sum_{t=1}^{T}\E[(C.E)_t] 
        \\ 
        \leq \frac{2(f(\Bar{X}_1) - f^{*} + \phi_1(C.E)_{1})}{\eta (1 - 2L \eta) T} + \frac{L \eta}{n(1 - 2 L \eta)}\sigma^2 \\ + \frac{L^2(1-6L\eta)}{T(1-2L \eta)}\sum_{t=1}^{T}\frac{\gamma_t}{\rho_{t}}
     + \frac{L \eta}{(1-2 L \eta)}\frac{1}{nT}\sum_{t=1}^{T}\sum_{i=1}^{n}{D}^2_{t,i}
    \end{multline}
\end{theorem}

\begin{proof}
We start by defining the notations we use in the theoretical analysis of the algorithm.\\
\textbf{Vectors:}
\begin{gather}
    X_{t}^{i} : \text{Parameter of client $i$ at time $t$}
    \\
    \nabla_{t,i} = \nabla f(X_{t}^{i}) \text{  and  } \widetilde \nabla_{t,i} = \widetilde \nabla f(X_{t}^{i}, \xi_{t}^{i})\\
    \Bar{X}_t = \frac{1}{n}\sum_{i=1}^{n}X_{t}^{i}\\
    \Bar{\nabla}_t = \nabla f(\Bar{X}_t)\\
    \Bar{\delta}_t = \frac{1}{n}\sum_{i=1}^{n}\delta_{t}^{i}
\end{gather}
\textbf{Matrix form:}
\begin{align}
    X^t = [X_{t}^{1}, X_{t}^{2},\cdots, X_{t}^{n}] \in \mathbb{R}^{d \times n}
    \\
    \Bar{X}^t = [\Bar{X}_t, \Bar{X}_t,\cdots, \Bar{X}_t] \in \mathbb{R}^{d \times n}
    \\
    \nabla^t = [\nabla_{t,1}, \nabla_{t,2}, \cdots, \nabla_{t,n}] \in \mathbb{R}^{d \times n}
    \\
    \widetilde \nabla^t = [\widetilde \nabla_{t,1}, \widetilde \nabla_{t,2}, \cdots, \widetilde \nabla_{t,n}] \in \mathbb{R}^{d \times n}
    \\
    \Bar{\nabla}^t = [\Bar{\nabla}_{t}, \Bar{\nabla}_{t}, \cdots, \Bar{\nabla}_{t}] \in \mathbb{R}^{d \times n}
    \\
    \delta^t = [\delta_{t}^1, \delta_{t}^2, \cdots, \delta_{t}^n ] \in \mathbb{R}^{d \times n}
    \\
    \Bar{\delta}^{t} = [\Bar{\delta}_{t}, \Bar{\delta}_{t}, \cdots, \Bar{\delta}_{t}] \in \mathbb{R}^{d \times n}
\end{align}
\clearpage
\section{Proof for FedNDL1}
\begin{align}
    \label{eq:DFL-SGD}
    X_{t+\frac{1}{2}} &= X_{t} - \eta \widetilde \nabla_{t,i}
    \\
    \label{eq: DFL-NoisyConsensus}
    X_{t+1} &= \Big(X_{t+\frac{1}{2}} + \delta_{t,i}\Big)W
    \\
    \label{eq: DFL-AvgSGD}
    \Bar{X}_{t+1} &= \Bar{X_{t}} - \frac{\eta}{n}\sum_{i=1}^{n} \widetilde \nabla_{t,i} + \Bar{\delta_t}
\end{align}
\begin{align}
    \label{eq:inner-a-b}
    \langle a_1, a_1 \rangle = \frac{1}{2}(\|a_1\|^2 + \|a_2\|^2 - \|a_1 - a_2\|^2) 
\end{align}
Using the $L-$smoothness assumption we get,
\begin{align}
    f(\Bar{X}_{t+1}) &\leq f(\Bar{X_{t}}) + \langle \nabla f(\Bar{x}_t), \Bar{X}_{t+1} - \Bar{X_{t}} \rangle + \frac{L}{2}\|\Bar{X}_{t+1} - \Bar{X}_{t}\|^2
    \\
    \label{eq: DFL-L-Smoothness}
    &\leq f(\Bar{X_{t}}) + \langle \nabla f(\Bar{x}_t), -\frac{\eta}{n}\sum_{i=1}^{n} \widetilde \nabla_{t,i} + \Bar{\delta_t} \rangle + \frac{L}{2}\|-\frac{\eta}{n}\sum_{i=1}^{n} \widetilde \nabla_{t,i} + \Bar{\delta_t}\|^2
\end{align}
Now, taking expectation w.r.t data and alongside the zero mean the assumption of noise in \cref{eq: DFL-L-Smoothness} we get
\begin{align}
\label{eq: DFL-Exp-L-Smoothness}
    \E [f(\Bar{X}_{t+1})] \leq \E[f(\Bar{X_{t}})] + \underbrace{\E[\langle \nabla f(\Bar{x}_t), -\frac{\eta}{n}\sum_{i=1}^{n} \widetilde \nabla_{t,i} \rangle]}_{Term\ 1} + \underbrace{\frac{L}{2}\E[\|-\frac{\eta}{n}\sum_{i=1}^{n} \widetilde \nabla_{t,i} + \Bar{\delta_t}\|^2]}_{Term\ 2}
\end{align}
Starting with $Term\ 1:$
\begin{align}
\E[\langle \nabla f(\Bar{x}_t), -\frac{\eta}{n}\sum_{i=1}^{n} \widetilde \nabla_{t,i} \rangle] & = -\eta \E[\langle \Bar{\nabla}_t, \frac{1}{n}\sum_{i=1}^{n} \widetilde \nabla_{t,i} \rangle]
\\
& = -\eta \E[\langle \Bar{\nabla}_t, \frac{1}{n}\sum_{i=1}^{n}( \widetilde \nabla_{t,i} +  \nabla_{t,i} - \nabla_{t,i})\rangle]
\\
\label{eq:DFL-Term-1}
& = -\eta \E[\langle \Bar{\nabla}_t, \frac{1}{n}\sum_{i=1}^{n}\nabla_{t,i}\rangle] -\eta \E[\langle \Bar{\nabla}_t, \frac{1}{n}\sum_{i=1}^{n}( \widetilde \nabla_{t,i} - \nabla_{t,i})\rangle]
\end{align}
The second term in \cref{eq:DFL-Term-1} will be zero, since $\E[\widetilde \nabla_{t,i}] = \E[\nabla_{t,i}].$
Now focusing on first term in \cref{eq:DFL-Term-1} and using \cref{eq:inner-a-b}, we have
\begin{align}
    \E[\langle \nabla f(\Bar{x}_t), -\frac{\eta}{n}\sum_{i=1}^{n} \widetilde \nabla_{t,i} \rangle] &= -\eta \E[\langle \Bar{\nabla}_t, \frac{1}{n}\sum_{i=1}^{n}\nabla_{t,i}\rangle]
    \\
    \label{eq:DFL-Term-1-Sol}
    & = -\frac{\eta}{2}\E\Big[\|\Bar \nabla_{t}\|^2 + \|\frac{1}{n}\sum_{i=1}^{n} \nabla_{t,i}\|^2 - \|\Bar \nabla_{t} - \frac{1}{n}\sum_{i=1}^{n} \nabla_{t,i}\|^2 \Big]
\end{align}
Now, using $Term\ 2:$
\begin{align}
    \frac{L}{2}\E[\|\Bar{\delta_t} -\frac{\eta}{n}\sum_{i=1}^{n} \widetilde \nabla_{t,i}\|^2]
    & =  \frac{L}{2}(\E[\|\Bar{\delta_t}\|^2] + \E[\|\frac{\eta}{n}\sum_{i=1}^{n}\widetilde \nabla_{t,i}\|^2] -\E[\langle \Bar{\delta_t}, \frac{\eta}{n}\sum_{i=1}^{n}\widetilde \nabla_{t,i}\rangle])
\end{align}
Using Young's inequality above, we have
\begin{align}
\label{eq:DFL-Term-2}
    \frac{L}{2}\E[\|\Bar{\delta_t} -\frac{\eta}{n}\sum_{i=1}^{n} \widetilde \nabla_{t,i}\|^2]
    & = L\Big(\underbrace{\E[\|\Bar{\delta_t}\|^2]}_{P} + \underbrace{\eta^2 \E[\|\frac{1}{n}\sum_{i=1}^{n}\widetilde \nabla_{t,i}\|^2]}_{Q}\Big)
\end{align}
Solving $P:$
\begin{align}
    \E[\|\Bar{\delta_t}\|^2] & = \E[\|\frac{1}{n}\sum_{i=1}^{n}\delta_{t,i}\|^2]
    \\
    \label{eq:DFL-Term-2-Jensen}
    & \leq \frac{1}{n}\sum_{i=1}^{n}\E[\|\delta_{t,i}\|^2]
    \\
    \label{eq:DFL-Term-2-P}
    & \leq \frac{1}{n}\sum_{i=1}^{n}{D}^2_{t,i}
\end{align}
Here, \cref{eq:DFL-Term-2-Jensen} follows due to Jensen's Inequality and \cref{eq:DFL-Term-2-P} follows due to Assumption \ref{as6}.
Solving $Q:$
\begin{align}
    \label{eq:see-prof-proof}
    \eta^2 \E[\|\frac{1}{n}\sum_{i=1}^{n}\widetilde \nabla_{t,i}\|^2] &= \eta^2 \E[\|\frac{1}{n}\sum_{i=1}^{n}(\widetilde \nabla_{t,i} + \nabla_{t,i} - \nabla_{t,i})\|^2]
    \\
    \label{eq:DFL-Term-2-Q}
    &\leq \eta^2\Big[\frac{\sigma^2}{n} + \E[\|\frac{1}{n}\sum_{i=1}^{n}\nabla_{t,i}\|^2]\Big]
\end{align}
Now, combining the results from \cref{eq:DFL-Term-1-Sol}, \cref{eq:DFL-Term-2-P}, \cref{eq:DFL-Term-2-Q}, and putting it back in \cref{eq: DFL-Exp-L-Smoothness}, we get
\begin{align}
    \nonumber
    \E [f(\Bar{X}_{t+1})] \leq \E[f(\Bar{X_{t}})] - \frac{\eta}{2}\E[\|\Bar \nabla_{t}\|^2] - \underbrace{\frac{\eta}{2}\E[\|\frac{1}{n}\sum_{i=1}^{n} \nabla_{t,i}\|^2]}_{M} + \frac{\eta}{2}\E[\|\Bar \nabla_{t} - \frac{1}{n}\sum_{i=1}^{n} \nabla_{t,i}\|^2] + \frac{L}{n}\sum_{i=1}^{n}{D}^2_{t,i} 
    \\
    \label{eq: DFL-Exp-L-Smoothness-pro}
    + \frac{L \eta^2}{n}\sigma^2 + \underbrace{L \eta^2 \E[\|\frac{1}{n}\sum_{i=1}^{n}\nabla_{t,i}\|^2]}_{N}
\end{align}
In the equation above reducing $N$ as follows,
\begin{align}
    N &= L \eta^2 \E[\|\frac{1}{n}\sum_{i=1}^{n}\nabla_{t,i}\|^2]
    \\
    & = L \eta^2 \E[\|\frac{1}{n}\sum_{i=1}^{n}\nabla_{t,i} - \Bar{\nabla_{t}}+\Bar{\nabla_{t}}\|^2]
    \\
    \label{eq:DFL-Term-N-Young's}
    &\leq 2L \eta^2 \E[\|\Bar{\nabla_{t}}\|^2] + 2L \eta^2 \E[\|\frac{1}{n}\sum_{i=1}^{n}\nabla_{t,i} - \Bar{\nabla_{t}}\|^2]
\end{align}
Here, \cref{eq:DFL-Term-N-Young's} follows due to Young's Inequality. So, dropping $M$ and putting \cref{eq:DFL-Term-N-Young's} back in \cref{eq: DFL-Exp-L-Smoothness-pro}, we have

\begin{align}
\nonumber
    \E [f(\Bar{X}_{t+1})] \leq \E[f(\Bar{X_{t}})] - \eta(\frac{1}{2} - 2L\eta)\E[\|\Bar \nabla_{t}\|^2] + \frac{L}{n}\sum_{i=1}^{n}{D}^2_{t,i} + \frac{L \eta^2}{n}\sigma^2 \\+ (\frac{\eta}{2} + 2L\eta^2) \E[\|\frac{1}{n}\sum_{i=1}^{n}\nabla_{t,i} - \Bar{\nabla_{t}}\|^2]
    \\ \nonumber
    \eta(\frac{1}{2} - 2L\eta)\E[\|\Bar \nabla_{t}\|^2] \leq (\E[f(\Bar{X_{t}})] - \E [f(\Bar{X}_{t+1})]) +  \frac{L}{n}\sum_{i=1}^{n}{D}^2_{t,i} + \frac{L \eta^2}{n}\sigma^2 
    \\ \label{}
    + (\frac{\eta}{2} + 2L\eta^2) \underbrace{\E[\|\frac{1}{n}\sum_{i=1}^{n}\nabla_{t,i} - \Bar{\nabla_{t}}\|^2]}_{C}
\end{align}
We need to study C:
\begin{align}
    C &= \E[\|\frac{1}{n}\sum_{i=1}^{n}\nabla_{t,i} - \Bar{\nabla_{t}}\|^2]
    \\
    \label{eq:DFL-Term-C}
    & \leq L^2 \underbrace{\frac{1}{n}\sum_{i=1}^{n}\E[\|X_{t,i} - \Bar{X}_{t}\|^2]}_{Consensus\ Error}
\end{align}
Now, let us study the consensus error at $t+1$. So, we reformulate the consensus error by using the Frobenius norm, we start with,
\begin{align}
    \frac{1}{n}\E\Big[\Big\|\Bar{X}^{t+1} - {X}^{t+1}\Big\|^2_{F}\Big] &= \frac{1}{n}\E\Big[\Big\|\Big[\Bar{X}^{t} - \eta \widetilde \nabla^{t} \frac{\mathds{1}\mathds{1}^{\top}}{n} + \Bar{\delta}^{t} - {X}^{t} + \eta \widetilde \nabla^{t} - \delta^{t}\Big]W\Big\|^2_{F}\Big]
    \\
    & = \frac{1}{n}\E\Big[\Big\|\Big[\Bar{X}^{t} - {X}^{t} - \eta (\widetilde \nabla^{t} \frac{\mathds{1}\mathds{1}^{\top}}{n} - \widetilde \nabla^{t}) + (\Bar{\delta}^{t} -\delta^{t})\Big]W\Big\|^2_{F}\Big]
    \\
    \nonumber
    &\leq \frac{(1-\rho)}{n}\E\Big[\Big\|\Big[\Bar{X}^{t} - {X}^{t} - \eta (\widetilde \nabla^{t} \frac{\mathds{1}\mathds{1}^{\top}}{n} - \widetilde \nabla^{t}) + (\Bar{\delta}^{t} -\delta^{t}) + \eta (\nabla^{t} \frac{\mathds{1}\mathds{1}^{\top}}{n} - \nabla^{t}) 
    \\
    \label{eq:DFL-Consensus-error}
    - \eta (\nabla^{t} \frac{\mathds{1}\mathds{1}^{\top}}{n} - \nabla^{t})\Big]\Big\|^2_{F}\Big]
\end{align}
Let, $a = (\Bar{X}^{t} - {X}^{t} - \eta (\nabla^{t} \frac{\mathds{1}\mathds{1}^{\top}}{n} - \nabla^{t}))$, $b = \eta((\nabla^{t} \frac{\mathds{1}\mathds{1}^{\top}}{n} - \nabla^{t})- (\widetilde \nabla^{t} \frac{\mathds{1}\mathds{1}^{\top}}{n} - \widetilde \nabla^{t}))$, and $c = (\Bar{\delta}^{t} -\delta^{t})$. Using this in \cref{eq:DFL-Consensus-error}, we have
\begin{align}
    \frac{1}{n}\E\Big[\Big\|\Bar{X}^{t+1} - {X}^{t+1}\Big\|^2_{F}\Big] \leq \frac{(1-\rho)}{n}\E\Big[\big\|a\big\|^2_{F} + \big\|b\big\|^2_{F} + \big\|c\big\|^2_{F} + 2 \langle a, b \rangle + 2 \langle b, c \rangle + 2 \langle c, a \rangle\Big]
\end{align}
In the result above, $\E[\langle a, b \rangle] = 0$, $\E[\langle b, c \rangle] = 0$, and $\E[\langle c, a \rangle] = 0$.
\begin{align}
\nonumber
    \frac{1}{n}\E\Big[\Big\|\Bar{X}^{t+1} - {X}^{t+1}\Big\|^2_{F}\Big] &\leq \frac{(1-\rho)}{n}\E\Big[\big\|a\big\|^2_{F} + \big\|b\big\|^2_{F} + \big\|c\big\|^2_{F}\Big]
    \\
    \label{eq:DFL-Consensus-Frob}
    \frac{1}{n}\E\Big[\Big\|\Bar{X}^{t+1} - {X}^{t+1}\Big\|^2_{F}\Big] &\leq \underbrace{\frac{(1-\rho)}{n}\E\Big[\big\|a\big\|^2_{F}\Big]}_{J} + \underbrace{\frac{(1-\rho)}{n}\E\Big[\big\|b\big\|^2_{F}\Big]}_{K} + \underbrace{\frac{(1-\rho)}{n}\E\Big[\big\|c\big\|^2_{F}\Big]}_{L}
\end{align}
Solving for Term $J:$
\begin{align}
\nonumber
    \frac{(1-\rho)}{n}\E\Big[\big\|a\big\|^2_{F}\Big] &= \frac{(1-\rho)}{n}\E\Big[\big\|\Bar{X}^{t} - {X}^{t} - \eta (\nabla^{t} \frac{\mathds{1}\mathds{1}^{\top}}{n} - \nabla^{t})\big\|^2_{F}\Big]
    \\
    \label{eq:DFL-Term-J}
    &= \frac{(1-\rho)(1+\frac{\rho}{2})}{n}\E\Big[\big\|\Bar{X}^{t} - {X}^{t}\big\|^2_{F}\Big] + \frac{\eta^2 (1-\rho)(1+\frac{2}{\rho})}{n}\underbrace{\E\Big[\big\|(\nabla^{t} \frac{\mathds{1}\mathds{1}^{\top}}{n} - \nabla^{t})\big\|^2_{F}\Big]}_{J_1}
\end{align}
Solving for Term $J_1:$
\begin{align}
\nonumber
    J_1 &= \sum_{i=1}^{n}\E\Big[\big\|\frac{1}{n}\sum_{j=1}^{n}\nabla f_j(x_t^j) - \nabla f_i(x_t^i) + \nabla f(\Bar{X}_t) - \nabla f(\Bar{X}_t)\big\|^2\Big]
    \\
    \label{eq:DFl-Term-J-1-Youngs}
    & \leq 2 \sum_{i=1}^{n}\E\Big[\big\| \frac{1}{n}\sum_{j=1}^{n}(\nabla f_j(x_t^j) - f_j(\Bar{X}_t))\big\|^2\Big] + 2 \sum_{i=1}^{n}\E\Big[\big\| \nabla f_i(x_t^i) - \nabla f(\Bar{X}_t)\big\|^2\Big]
    \\
    \label{eq:DFL-Term-J-1-jensen}
    & \leq 2 \sum_{j=1}^{n}\E\Big[\big\| \nabla f_j(x_t^j) - f_j(\Bar{X}_t)\big\|^2\Big] + 2 \sum_{i=1}^{n}\E\Big[\big\| \nabla f_i(x_t^i) - \nabla f(\Bar{X}_t) + \nabla f(x_t^i) - \nabla f(x_t^i)\big\|^2\Big]
    \\
    \label{eq:DFL-Term-J-1-Young's-Smooth}
    & \leq 2L^2 \E\Big[\big\|\Bar{X}_t - X^t\big\|^2_{F}\Big] + 4 \sum_{i=1}^{n}\E\Big[\big\| \nabla f_i(x_t^i) - \nabla f(x_t^i)\big\|^2\Big] + 4 \sum_{i=1}^{n}\E\Big[\big\| \nabla f(\Bar{X}_t) - \nabla f(x_t^i)\big\|^2\Big]
    \\
    \label{eq:DFL-Term-J-1-Smooth-BCD}
    & \leq 6L^2 \E\Big[\big\|\Bar{X}_t - X^t\big\|^2_{F}\Big] + 4 n B^2
\end{align}
In \cref{eq:DFl-Term-J-1-Youngs} follows due to Young's inequality and \cref{eq:DFL-Term-J-1-jensen} follows due to Jensen's inequality. Again, \cref{eq:DFL-Term-J-1-Young's-Smooth} follows due to Young's inequality and the Frobenius norm. Similarly, \cref{eq:DFL-Term-J-1-Smooth-BCD} holds due to $L-$smoothness and assumption \ref{as5}. So, putting the result obtained for $J_1$ in \cref{eq:DFL-Term-J}, Term $J$ becomes,
\begin{align}
    J \leq \frac{(1-\rho)(1+\frac{\rho}{2})}{n}\E\Big[\big\|\Bar{X}^{t} - {X}^{t}\big\|^2_{F}\Big] + \frac{\eta^2 (1-\rho)(1+\frac{2}{\rho})}{n}\Big(6L^2 \E\Big[\big\|\Bar{X}_t - X^t\big\|^2_{F}\Big] + 4 n B^2\Big)
\end{align}
Now, using the fact that $\frac{(1-\rho)(1+\frac{\rho}{2})}{n} \leq \frac{(1 - \frac{\rho}{2})}{n}$ and $\frac{\eta^2 (1-\rho)(1+\frac{2}{\rho})}{n} \leq \frac{2\eta^2}{n \rho}$ in the equation above, we get
\begin{align}
    \label{eq:DFL-Term-J-Sol}
    J \leq \frac{1}{n}\Big[(1 - \frac{\rho}{2}) + \frac{12 L^2 \eta^2}{\rho} \Big]\E\Big[\big\|\Bar{X}^{t} - {X}^{t}\big\|^2_{F}\Big] + \frac{8 \eta^2 B^2}{\rho}
\end{align}
Now, solving Term $K:$
\begin{align}
    K &= \frac{(1-\rho)}{n}\E\Big[\big\| \eta\Big((\nabla^{t} \frac{\mathds{1}\mathds{1}^{\top}}{n} - \nabla^{t})- (\widetilde \nabla^{t} \frac{\mathds{1}\mathds{1}^{\top}}{n} - \widetilde \nabla^{t})\Big) \big\|^2_{F}\Big]
    \\
    \label{eq:DFL-Term-K-sol}
    K &\leq (1-\rho) \eta^2 \sigma^2
\end{align}
In \cref{eq:DFL-Term-K-sol} follows due to Assumption \ref{as3}.
Now, solving Term $L:$
\begin{align}
    L &= \frac{(1-\rho)}{n}\E\big[\big\| \Bar{\delta}^{t} -\delta^{t} \big\|^2_{F}\big]
    \\
    \label{eq:DFL-Term-L-Young's}
    & \leq \frac{(1-\rho)(1+\frac{\rho}{2})}{n}\E\big[\big\| \Bar{\delta}^{t} \|^2_{F}\big] + \frac{(1-\rho)(1+\frac{2}{\rho})}{n}\E\big[\big\|\delta^{t} \big\|^2_{F}\big]
    \\
    & \leq \frac{(1-\rho)(1+\frac{\rho}{2})}{n} \sum_{j=1}^{n} \E\big[\big\| \Bar{\delta}^{t} \|^2\big] + \frac{ (1-\rho)(1+\frac{2}{\rho})}{n}\sum_{j=1}^{n}\E\big[\big\|\delta_{t}^{j} \big\|^2\big]
    \\
    & \leq \frac{(1-\rho)(1+\frac{\rho}{2})}{n} n \E\big[\big\| \frac{1}{n}\sum_{i=1}^{n}\delta_{t}^{i} \|^2\big] + \frac{(1-\rho)(1+\frac{2}{\rho})}{n}\sum_{i=1}^{n}{D}^2_{t,i}
    \\
    \label{eq:DFL-Term-L-jensen}
    & \leq \frac{(1-\rho)(1+\frac{\rho}{2})}{n} \sum_{i=1}^{n}\E\big[\big\|\delta_{t}^{i} \big\|^2\big] + \frac{(1-\rho)(1+\frac{2}{\rho})}{n}\sum_{i=1}^{n}{D}^2_{t,i}
    \\
    & \leq \frac{(1-\rho)(1+\frac{\rho}{2})}{n} \sum_{i=1}^{n}{D}^2_{t,i} + \frac{(1-\rho)(1+\frac{2}{\rho})}{n}\sum_{i=1}^{n}{D}^2_{t,i}
    \\
    & \leq \frac{1}{n}(1-\rho)\Big[2+\frac{\rho}{2}+\frac{2}{\rho}\Big]\sum_{i=1}^{n}{D}^2_{t,i}
    \\
    & \leq \frac{1}{n}\Big[\frac{\rho}{2}+\frac{2}{\rho}\Big]\sum_{i=1}^{n}{D}^2_{t,i}
\end{align}
So, using the results for $J$, $K$, and $L$ in \cref{eq:DFL-Consensus-Frob}, we get
\begin{align}
    \nonumber
    \frac{1}{n}\E\Big[\Big\|\Bar{X}^{t+1} - {X}^{t+1}\Big\|^2_{F}\Big] & \leq J + K + L
    \\
    \nonumber
    & \leq \frac{1}{n}\Big[(1 - \frac{\rho}{2}) + \frac{12 L^2 \eta^2}{\rho} \Big]\E\Big[\big\|\Bar{X}^{t} - {X}^{t}\big\|^2_{F}\Big] + \frac{8 \eta^2 B^2}{\rho} + (1-\rho) \eta^2 \sigma^2 
    \\
    \label{eq:DFL-Consensus-Sol}
    &+ \frac{1}{n}\Big[\frac{\rho}{2}+\frac{2}{\rho}\Big]\sum_{i=1}^{n}{D}^2_{t,i}
\end{align}
From here on we are going to use the shorthand, $(C.E)_{t}$ for $\frac{1}{n}\Big\|\Bar{X}^{t} - {X}^{t}\Big\|^2_{F}$. So, using the shorthand in \cref{eq:DFL-Consensus-Sol}, we have
\begin{align}
   \E[(C.E)_{t+1}] \leq \Big[(1 - \frac{\rho}{2}) + \frac{12 L^2 \eta^2}{\rho} \Big]\E[(C.E)_{t}] + \frac{8 \eta^2 B^2}{\rho} + (1-\rho) \eta^2 \sigma^2 + \frac{1}{n}\Big[\frac{\rho}{2}+\frac{2}{\rho}\Big]\sum_{i=1}^{n}{D}^2_{t,i}
\end{align}
Let there be a potential function $\psi_{t}$, defined as,
\begin{equation}
    \label{eq:DFL-Potential}
    \psi_{t} = f(\Bar{X}_{t}) + \phi \E[(C.E)_{t}], {\text{  where $\phi$ is a constant.}}
\end{equation}
Now, we use the potential function to complete the proof. 
\begin{align}
\label{eq:DFL-potential-solve-start}
   \psi_{t+1} - \psi_{t} = \big\{f(\Bar{X}_{t+1}) - f(\Bar{X}_{t})\big\} + \phi \big\{\E[(C.E)_{t+1}] - \E[(C.E)_{t}]\big\}
   \\
   \nonumber
   \psi_{t+1} - \psi_{t} \leq -\frac{\eta}{2}(1 - 4\eta L)\E[\|\Bar \nabla_{t}\|^2] + \frac{L}{n}\sum_{i=1}^{n}{D}^2_{t,i} + \frac{L \eta^2}{n}\sigma^2 + \phi\frac{8 \eta^2 B^2}{\rho} + \phi(1-\rho) \eta^2 \sigma^2
   \\
       \label{eq:DFL-potential-solve}
    + \Big[L^2 \eta (\frac{1}{2} + 2L\eta) 
     - \phi\Big(\frac{\rho}{2} - \frac{12 L^2 \eta^2}{\rho} \Big)\Big]\E[(C.E)_{t}]  + \frac{\phi}{n}\Big[\frac{\rho}{2}+\frac{2}{\rho}\Big]\sum_{i=1}^{n}{D}^2_{t,i}
\end{align}
For $\phi > \frac{L^2 \eta (\frac{1}{2} + 2L\eta)} {\frac{\rho}{2} - \frac{12 L^2 \eta^2}{\rho}}$ and $\eta L < \frac{\rho}{2\sqrt{6}}$:
\begin{align}
\nonumber
   \E[\|\Bar \nabla_{t}\|^2] + \frac{2\Big[ \Big(\phi(\frac{\rho}{2} - \frac{12 L^2 \eta^2}{\rho}) - L^2 \eta (\frac{1}{2} + 2L\eta) \Big)\Big]}{\eta(1 - 4\eta L)}\E[(C.E)_{t}]  \leq \frac{2(\psi_{t} - \psi_{t+1})}{\eta(1 - 4\eta L)}
   \\
   \label{eq:DFL-potential-solve-1}
   + \frac{\eta(\phi(1 - \rho)+\frac{L}{n})}{(1 - 4\eta L)}\sigma^2 + \frac{16 \phi \eta B^2}{\rho(1 - 4\eta L)} + \frac{2\Big(L+ \phi\Big[\frac{\rho}{2}+\frac{2}{\rho}\Big]\Big)}{\eta(1 - 4\eta L)}\frac{1}{n}\sum_{i=1}^{n}{D}^2_{t,i}
\end{align}
Now let us assume that $\frac{2\Big[ \Big(\phi(\frac{\rho}{2} - \frac{12 L^2 \eta^2}{\rho}) - L^2 \eta (\frac{1}{2} + 2L\eta) \Big)\Big]}{\eta(1 - 4\eta L)} = C$ and we solve for $\phi$:
\begin{align}
    \phi = \frac{C\eta(1 - 4\eta L) + L^2 \eta (\frac{1}{2} + 4L\eta)}{\frac{\rho}{2} - \frac{12 L^2 \eta^2}{\rho}}
\end{align}
So, for $C = L^2$, we get $\phi = \frac{2 \rho L^2 \eta}{\rho^2 - 24 \eta^2L^2}$.

Now, summing \cref{eq:DFL-potential-solve-1} for $t=\{1 \cdots T\}$ and dividing it by $T$, we get
\begin{align}
\nonumber
   \frac{1}{T}\sum_{t=1}^{T}\E[\|\Bar \nabla_{t}\|^2] + \frac{2\Big[ \Big(\phi(\frac{\rho}{2} - \frac{12 L^2 \eta^2}{\rho}) - L^2 \eta (\frac{1}{2} + 2L\eta) \Big)\Big]}{\eta(1 - 4\eta L)}\frac{1}{T}\sum_{t=1}^{T}\E[(C.E)_{t}]  \leq \frac{2\frac{1}{T}\sum_{t=1}^{T}(\psi_{t} - \psi_{t+1})}{\eta(1 - 4\eta L)}
   \\
   \label{eq:DFL-potential-telescope}
   + \frac{\eta(\phi(1 - \rho)+\frac{L}{n})}{(1 - 4\eta L)}\sigma^2 + \frac{16 \phi \eta B^2}{\rho(1 - 4\eta L)} + \frac{2\Big(L+ \phi\Big[\frac{\rho}{2}+\frac{2}{\rho}\Big]\Big)}{\eta(1 - 4\eta L)}\frac{1}{nT}\sum_{t=1}^{T}\sum_{i=1}^{n}{D}^2_{t,i}
\end{align}
Also, telescoping over \cref{eq:DFL-potential-solve-start} and dividing it by $T$, we can write
\begin{equation}
    \label{eq:DFL-potential-start-telescope}
    \frac{\psi_{T+1} - \psi_{1}}{T} \geq \frac{f^{*} - f(\Bar{X}_1) - \phi(C.E)_{1}}{T}
\end{equation}
We use \cref{eq:DFL-potential-start-telescope} in \cref{eq:DFL-potential-telescope}, we get
\begin{multline}
\label{eq:DFL-potential-telescope-solve}
     \frac{1}{T}\sum_{t=1}^{T}\E[\|\nabla f(\Bar{X}_t)\|^2] + 
        \frac{L^2}{2}\frac{1}{T}\sum_{t=1}^T\E[(C.E)_t]
        \\
        \leq \frac{2(f(\Bar{X}_1) - f^{*} + \phi(C.E)_{1})}{\eta(1 - 4 \eta L)} + \frac{2\eta(\phi(1-\rho)+\frac{L}{n})}{(1 - 4 \eta L)}\sigma^2 
        \\
        + \frac{16 \phi\eta B^2}{(1 - 4 \eta L)\rho} 
     + \frac{2\Big(L+\phi\big[\frac{\rho}{2}+\frac{2}{\rho}\big]\Big)}{\eta(1 - 4 \eta L)}\frac{1}{nT}\sum_{t=1}^{T}\sum_{i=1}^{n}{D}^2_{t,i}
\end{multline}
\end{proof}
\clearpage
\section{Proof for FedNDL2}
Start:

\begin{flalign}
\Bar{X}_{t+1} = \Bar{X}_{t} + \Bar{\delta}_{t} - \frac{\eta}{n}\sum_{i=1}^{n}\widetilde \nabla f_{i}(x_{t+\frac{1}{2}}^{(i)})
\end{flalign}
Now, using L-smoothness, we get
\begin{flalign}
\label{eq:start}
    f(\Bar{X}_{t+1}) \leq f(\Bar{X}_{t}) + \underbrace{\langle \nabla f(\Bar{X}_{t}),\Bar{X}_{t+1} - \Bar{X}_{t}\rangle}_{Term (A)} +\underbrace{\frac{L}{2} \|\Bar{X}_{t+1} - \Bar{X}_{t}\|^2}_{Term (B)}
\end{flalign}
Term (A):
\begin{align}
    \langle \nabla f(\Bar{X}_{t}),\Bar{X}_{t+1} - \Bar{X}_{t}\rangle & = \langle \nabla f(\Bar{X}_{t}),\Bar{\delta}_{t} - \frac{\eta}{n}\sum_{i=1}^{n}\widetilde \nabla f_{i}(x_{t+\frac{1}{2}}^{(i)})\rangle
    \\
    & = \underbrace{\langle \nabla f(\Bar{X}_{t}),\Bar{\delta}_{t} \rangle}_{Term(A_1)} + \underbrace{\langle \nabla f(\Bar{X}_{t}),- \frac{\eta}{n}\sum_{i=1}^{n}\widetilde \nabla f_{i}(x_{t+\frac{1}{2}}^{(i)})\rangle}_{Term(A_2)}
\end{align}
Term $A_1$:
\begin{align}
\label{eq:A_1-expec}
   \E[A_1] &= \E[\langle \nabla f(\Bar{X}_{t}),\Bar{\delta}_{t} \rangle] 
   \\
   &= 0
\end{align}
Here \cref{eq:A_1-expec} follows due to the zero mean of the noise.
In term $A_2:$ taking the expectation w.r.t to data, we get
\begin{align}
    \E[\langle \nabla f(\Bar{X}_{t}),- \frac{\eta}{n}\sum_{i=1}^{n}\widetilde \nabla f_{i}(x_{t+\frac{1}{2}}^{(i)})\rangle] = -\eta \E[\langle \nabla f(\Bar{X}_{t}), \frac{1}{n}\sum_{i=1}^{n}\nabla f_{i}(x_{t+\frac{1}{2}}^{(i)})\rangle]
\end{align}
Using the formulation, $\langle a_1, a_2\rangle = \frac{1}{2}\Big(\|a_1\|^2 + \|a_2\|^2 + \|a_1 - a_2\|^2\Big)$
\begin{align}
\E[A_2] = -\frac{\eta}{2}\E[\|\Bar{\nabla}_t\|^2] 
\underbrace{-\frac{\eta}{2}\E[\|\frac{1}{n}\sum_{i=1}^{n}\nabla f_{i}(x_{t+\frac{1}{2}}^{(i)})\|^2]}_{Term(A_3)} 
+\frac{\eta}{2}\E[\|\Bar{\nabla}_t - \frac{1}{n}\sum_{i=1}^{n}\nabla f_{i}(x_{t+\frac{1}{2}}^{(i)})\|^2]
\end{align}

Putting in the values of $A_1$ and $A_2$ back in term A we get,
\begin{align}
    \E[A] \leq -\frac{\eta}{2}\E[\|\Bar{\nabla}_t\|^2] -\frac{\eta}{2}\E[\|\frac{1}{n}\sum_{i=1}^{n}\nabla f_{i}(x_{t+\frac{1}{2}}^{(i)})\|^2] + \frac{\eta}{2}\E[\|\Bar{\nabla}_t - \frac{1}{n}f_i(x_{t+\frac{1}{2}}^i)\|^2]
\end{align}

Starting with Term B:
\begin{align}
    \frac{L}{2} \|\Bar{X}_{t+1} - \Bar{X}_{t}\|^2 &= \frac{L}{2} \|\Bar{\delta}_{t} - \frac{\eta}{n}\sum_{i=1}^{n}\widetilde \nabla f_{i}(x_{t+\frac{1}{2}}^{(i)})\|^2
    \\
    &\leq L[\|\Bar{\delta}_{t}\|^2 + \eta^2\|\frac{1}{n}\sum_{i=1}^{n}\widetilde \nabla f_{i}(x_{t+\frac{1}{2}}^{(i)})\|^2]
\end{align}
\begin{align}
    \E[B] &\leq L\E[\|\Bar{\delta}_{t}\|^2] + L\eta^2\E[\|\frac{1}{n}\sum_{i=1}^{n}\widetilde \nabla f_{i}(x_{t+\frac{1}{2}}^{(i)})\|^2]
    \\
    &\leq \frac{L}{n}\sum_{i=1}^nD_{t,i}^2 + L\eta^2\E[\|\frac{1}{n}\sum_{i=1}^{n}\big(\widetilde \nabla f_{i}(x_{t+\frac{1}{2}}^{(i)}) - \nabla f_{i}(x_{t+\frac{1}{2}}^{(i)}) + \nabla f_{i}(x_{t+\frac{1}{2}}^{(i)})\big)\|^2]
    \\
    \nonumber
    &\leq \frac{L}{n}\sum_{i=1}^nD_{t,i}^2 + L\eta^2\Big(\E[\|\frac{1}{n}\sum_{i=1}^{n}\big(\widetilde \nabla f_{i}(x_{t+\frac{1}{2}}^{(i)}) - \nabla f_{i}(x_{t+\frac{1}{2}}^{(i)})\big)\|^2] + \E[\|\frac{1}{n}\sum_{i=1}^{n}\nabla f_{i}(x_{t+\frac{1}{2}}^{(i)}))\|^2\Big)] \\&+ 2\E[\langle\frac{1}{n}\sum_{i=1}^{n}\big(\widetilde \nabla f_{i}(x_{t+\frac{1}{2}}^{(i)}) - \nabla f_{i}(x_{t+\frac{1}{2}}^{(i)})\big), \frac{1}{n}\sum_{i=1}^{n}\nabla f_{i}(x_{t+\frac{1}{2}}^{(i)})\rangle]
    \\
    &\leq \frac{L}{n}\sum_{i=1}^nD_{t,i}^2 + \frac{L\eta^2 \sigma^2}{n} + L \eta^2\underbrace{\E[\|\frac{1}{n}\sum_{i=1}^{n}\nabla f_{i}(x_{t+\frac{1}{2}}^{(i)}))\|^2]}_{Term(B_1)}
\end{align}
% ******************************************************

% Term $B_1:$
% \begin{align}
%    \E[\|\frac{1}{n}\sum_{i=1}^{n}\nabla f_{i}(x_{t+\frac{1}{2}}^{(i)})\|^2] = \E\Big[\|\frac{1}{n}\sum_{i=1}^{n}\Big(\nabla f_{i}(x_{t+\frac{1}{2}}^{(i)}) - \nabla f_{i}\big(x_{t+\frac{1}{2}}^{(i)} - \sum_{j=1}^{n}w_{ij}\delta_t^j\big) + \nabla f_{i}\big(x_{t+\frac{1}{2}}^{(i)}- \sum_{j=1}^{n}w_{ij}\delta_t^j\big)\Big)\|^2\Big]
%     \\
%     \leq 2\E\Big[\|\frac{1}{n}\sum_{i=1}^{n}\Big(\nabla f_{i}(x_{t+\frac{1}{2}}^{(i)}) - \nabla f_{i}\big(x_{t+\frac{1}{2}}^{(i)} - \sum_{j=1}^{n}w_{ij}\delta_t^j\big)\Big)\|^2 + \|\frac{1}{n}\sum_{i=1}^{n}\Big(+ \nabla f_{i}\big(x_{t+\frac{1}{2}}^{(i)}- \sum_{j=1}^{n}w_{ij}\delta_t^j\big)\Big)\|^2\Big] 
%     \\
%     \leq 2\E\Big[\|\frac{1}{n}\sum_{i=1}^{n}L^2\sum_{j=1}^{n}w_{ij}\|\delta_t^j\|^2+ \|\frac{1}{n}\sum_{i=1}^{n}\Big(\nabla f_{i}\big(x_{t+\frac{1}{2}}^{(i)}- \sum_{j=1}^{n}w_{ij}\delta_t^j\big)\Big)\|^2\Big] 
% \end{align}
% Now taking the expectation w.r.t noise, we get
% \begin{align}
% \label{eq:A_3-final}
%     L\eta^2\E[B_1] \leq \frac{2L^3 \eta^2}{n}\sum_{i=1}^{n}D_{t,i}^2 +2L\eta^2\E[\|\frac{1}{n}\sum_{i=1}^{n}\nabla f_{i}(\sum_{j=1}^{n}w_{ij}x_t^j)\|^2]
% \end{align}
% So, term B is
% \begin{align}
%     \E[B] \leq L(1+2L^2\eta^2)\frac{1}{n}\sum_{i=1}^nD_{t,i}^2 + \frac{L\eta^2 \sigma^2}{n} + 2L\eta^2\E[\|\frac{1}{n}\sum_{i=1}^{n}\nabla f_{i}(\sum_{j=1}^{n}w_{ij}x_t^j)\|^2]
% \end{align}
% ******************************************************

Taking the expectation of \cref{eq:start} w.r.t data and noise and then putting back term A and B back in it we get,

\begin{align}
\nonumber
   \E[f(\Bar{X}_{t+1})] &\leq  \E[f(\Bar{X}_{t})]-\frac{\eta}{2}\E[\|\Bar{\nabla}_t\|^2] 
-\frac{\eta}{2}\E[\|\frac{1}{n}\sum_{i=1}^{n}\nabla f_{i}(x_{t+\frac{1}{2}}^{(i)})\|^2] + L\eta^2\E[\|\frac{1}{n}\sum_{i=1}^{n}\nabla f_{i}(x_{t+\frac{1}{2}}^{(i)})\|^2]
\\
\nonumber
&+\frac{\eta}{2}\E[\|\Bar{\nabla}_t - \frac{1}{n}\sum_{i=1}^{n}\nabla f_{i}(x_{t+\frac{1}{2}}^{(i)})\|^2] + \frac{L\eta^2 \sigma^2}{n} + \frac{L}{n}\sum_{i=1}^n D_{t,i}^2
\\
\nonumber
& \leq \E[f(\Bar{X}_{t})]-\frac{\eta}{2}\E[\|\Bar{\nabla}_t\|^2] 
-\frac{\eta}{2}\E[\|\frac{1}{n}\sum_{i=1}^{n}\nabla f_{i}(x_{t+\frac{1}{2}}^{(i)})\|^2] + L\eta^2\E[\|\frac{1}{n}\sum_{i=1}^{n}\nabla f_{i}(x_{t+\frac{1}{2}}^{(i)}) - \Bar{\nabla}_t + \Bar{\nabla}_t\|^2]
\\
\nonumber
&+\frac{\eta}{2}\E[\|\Bar{\nabla}_t - \frac{1}{n}\sum_{i=1}^{n}\nabla f_{i}(x_{t+\frac{1}{2}}^{(i)})\|^2] + \frac{L\eta^2 \sigma^2}{n} + \frac{L}{n}\sum_{i=1}^n D_{t,i}^2
\\
\nonumber
& \leq \E[f(\Bar{X}_{t})]-\frac{\eta}{2}\E[\|\Bar{\nabla}_t\|^2] 
-\frac{\eta}{2}\E[\|\frac{1}{n}\sum_{i=1}^{n}\nabla f_{i}(x_{t+\frac{1}{2}}^{(i)})\|^2] + L\eta^2\E[\|\Bar{\nabla}_t - \frac{1}{n}\sum_{i=1}^{n}\nabla f_{i}(x_{t+\frac{1}{2}}^{(i)})\|^2] 
\\
\nonumber
&+ L\eta^2\E[\| \Bar{\nabla}_t\|^2]
+\frac{\eta}{2}\E[\|\Bar{\nabla}_t - \frac{1}{n}\sum_{i=1}^{n}\nabla f_{i}(x_{t+\frac{1}{2}}^{(i)})\|^2] + \frac{L\eta^2 \sigma^2}{n} + \frac{L}{n}\sum_{i=1}^n D_{t,i}^2
\\
\label{eq:Term-C-origin}
& \leq \E[f(\Bar{X}_{t})]-\frac{\eta}{2}(1 - 2L \eta)\E[\|\Bar{\nabla}_t\|^2] + (\frac{\eta}{2}+L\eta^2)\underbrace{\E[\|\Bar{\nabla}_t - \frac{1}{n}\sum_{i=1}^{n}\nabla f_{i}(x_{t+\frac{1}{2}}^{(i)})\|^2]}_{Term(C)} + \frac{L\eta^2 \sigma^2}{n} + \frac{L}{n}\sum_{i=1}^n D_{t,i}^2
\end{align}
% \begin{align}
%    \E[f(\Bar{X}_{t+1})] \leq \E[f(\Bar{X}_{t})] -\frac{\eta}{2}\E[\|\Bar{\nabla}_t\|^2] 
%  + \frac{L\eta^2 \sigma^2}{n} + \frac{L}{n}\sum_{i=1}^nD_{t,i}^2 + \frac{\eta}{2}\underbrace{\E[\|\Bar{\nabla}_t - \frac{1}{n}\sum_{i=1}^{n}\nabla f_{i}(x_{t+\frac{1}{2}}^{(i)})\|^2]}_{Term(C)}
% \end{align}
Term (C):
\begin{align}
    \E[\|\Bar{\nabla}_t - \frac{1}{n}\sum_{i=1}^{n}\nabla f_{i}(x_{t+\frac{1}{2}}^{(i)})\|^2] &= \E[\|\nabla f(\Bar{X}_t) - \frac{1}{n}\sum_{i=1}^{n}\nabla f_{i}(x_{t+\frac{1}{2}}^{(i)})\|^2]
    \\
    & = \E[\|\frac{1}{n}\sum_{i=1}^{n}(\nabla f_i(\Bar{X}_t) - \nabla f_{i}(x_{t+\frac{1}{2}}^{(i)}))\|^2]
    \\
    & \leq \frac{L^2}{n}\sum_{i=1}^{n}\E[\|\Bar{X}_t - x_{t+\frac{1}{2}}^{(i)}\|^2]
    \\
    & \leq \frac{L^2}{n}\E[\|\Bar{X}^t - X^{t+\frac{1}{2}}\|_{F}^2]
    \\
    & \leq \frac{L^2}{n}\E[\|\Bar{X}^t - X^{t}W - \delta^{t}W\|_{F}^2]
    \\
    & \leq \frac{2L^2}{n}\E[\|(\Bar{X}^t - X^{t})W\|_{F}^2] + \frac{2L^2}{n}\E[\|\delta^{t}W\|_{F}^2]
    \\
    & \leq \frac{2L^2(1-\rho)}{n}\E[\|\Bar{X}^t - X^{t}\|_{F}^2] + \frac{2L^2}{n}\sum_{i=1}^n D_{t,i}^2
\end{align}
Putting the value of the equation above in \cref{eq:Term-C-origin} and also denoting $\frac{1}{n}[\|\Bar{X}^t - X^{t}\|_{F}^2] = \E[(C.E)_t]$, we get
\begin{align}
\nonumber
    \E[f(\Bar{X}_{t+1})] &\leq \E[f(\Bar{X}_{t})] -\frac{\eta}{2}(1 - 2L \eta)\E[\|\Bar{\nabla}_t\|^2]
 + \frac{L\eta^2 \sigma^2}{n} + L^2[\eta(1+2 L \eta) + 2]\frac{1}{n}\sum_{i=1}^nD_{t,i}^2 
 \\
 \label{eq:consensus-origin}
 &+ L^2 \eta (1+2L \eta) (1-\rho) \E[(C.E)_t]
\end{align}
Now we start studying the consensus error at time $t+1$.
\begin{align}
    \frac{1}{n}\E[\|\Bar{X}^{t+1} - X^{t+1}\|_{F}^2] & = \frac{1}{n}\E[\|\Bar{X}^{t+\frac{1}{2}} - \eta \widetilde \nabla^t \frac{\mathds{1} \mathds{1}^{\top}}{n} - X^{t+\frac{1}{2}} + \eta \widetilde \nabla^t \|_{F}^2]
    \\
    & = \frac{1}{n}\E[\|\Bar{X}^{t} - \Bar{\delta}^{t} - \eta \widetilde \nabla^t \frac{\mathds{1} \mathds{1}^{\top}}{n} - (X^{t} + \delta^{t})W + \eta \widetilde \nabla^t \|_{F}^2]
    \\
    \nonumber
    & = \frac{1}{n}\E[\|(\Bar{X}^{t} - X^{t})W + (\Bar{\delta}^{t} - \delta^{t})W - \eta(\widetilde \nabla^t \frac{\mathds{1} \mathds{1}^{\top}}{n} - \widetilde \nabla^t)  + \eta (\nabla^{t} \frac{\mathds{1}\mathds{1}^{\top}}{n} - \nabla^{t}) 
    \\
    \label{eq:Term-Consensus-expansion}
    &- \eta (\nabla^{t} \frac{\mathds{1}\mathds{1}^{\top}}{n} - \nabla^{t}) \|_{F}^2]
\end{align}
Let, $a = ((\Bar{X}^{t} - {X}^{t})W - \eta (\nabla^{t} \frac{\mathds{1}\mathds{1}^{\top}}{n} - \nabla^{t}))$, $b = \eta((\nabla^{t} \frac{\mathds{1}\mathds{1}^{\top}}{n} - \nabla^{t})- (\widetilde \nabla^{t} \frac{\mathds{1}\mathds{1}^{\top}}{n} - \widetilde \nabla^{t}))$, and $c = ((\Bar{\delta}^{t} -\delta^{t})W)$. Using this in \cref{eq:Term-Consensus-expansion}, we have
\begin{align}
    \frac{1}{n}\E\Big[\Big\|\Bar{X}^{t+1} - {X}^{t+1}\Big\|^2_{F}\Big] \leq \frac{1}{n}\E\Big[\big\|a\big\|^2_{F} + \big\|b\big\|^2_{F} + \big\|c\big\|^2_{F} + 2 \langle a, b \rangle + 2 \langle b, c \rangle + 2 \langle c, a \rangle\Big]
\end{align}
In the result above, $\E[\langle a, b \rangle] = 0$, $\E[\langle b, c \rangle] = 0$, and $\E[\langle c, a \rangle] = 0$. 
\begin{align}
\nonumber
    \frac{1}{n}\E[\|\Bar{X}^{t+1} - X^{t+1}\|_{F}^2] & \leq \frac{1}{n}\E\Big[\big\|a\big\|^2_{F} + \big\|b\big\|^2_{F} + \big\|c\big\|^2_{F}\Big]
    \\
    \nonumber
    & \leq \frac{(1+\frac{\rho}{2})}{n}\E[\|(\Bar{X}^{t} - X^{t})W\|_{F}^2] + \frac{\eta^2(1+\frac{2}{\rho})}{n}\E[\|(\nabla^{t} \frac{\mathds{1}\mathds{1}^{\top}}{n} - \nabla^{t})\|^2_F] 
    \\
    &+ \frac{\eta^2}{n}\E[\|(\nabla^{t} \frac{\mathds{1}\mathds{1}^{\top}}{n} - \nabla^{t})- (\widetilde \nabla^{t} \frac{\mathds{1}\mathds{1}^{\top}}{n} - \widetilde \nabla^{t})\|^2_F] + \frac{1}{n}\E[\|(\Bar{\delta}^{t} -\delta^{t})W\|^2_F]
    \\
    \nonumber
    & \leq \frac{(1-\rho)(1+\frac{\rho}{2})}{n}\E[\|\Bar{X}^{t} - X^{t}\|_{F}^2] + \frac{\eta^2(1+\frac{2}{\rho})}{n}\underbrace{\E[\|(\nabla^{t} \frac{\mathds{1}\mathds{1}^{\top}}{n} - \nabla^{t})\|^2_F]}_{Term (C_1)} +  \eta^2 \sigma^2 
    \\
    \label{eq:C1-start}
    & + \frac{(1-\rho)}{n}\E[\|\Bar{\delta}^{t} -\delta^{t}\|^2_F]
\end{align}
% Using term $C_1:$
% \begin{align}
% \E[\|\widetilde \nabla^t \frac{\mathds{1} \mathds{1}^{\top}}{n} - \widetilde \nabla^t)\|^2_F] & =
% \E[\|(\widetilde \nabla^t \frac{\mathds{1} \mathds{1}^{\top}}{n} - \widetilde \nabla^t) - (\nabla^{t} \frac{\mathds{1}\mathds{1}^{\top}}{n} - \nabla^{t}) + (\nabla^{t} \frac{\mathds{1}\mathds{1}^{\top}}{n} - \nabla^{t})\|^2_F]
% \\
% & \leq 2 \E[\|(\widetilde \nabla^t \frac{\mathds{1} \mathds{1}^{\top}}{n} - \widetilde \nabla^t) - (\nabla^{t} \frac{\mathds{1}\mathds{1}^{\top}}{n} - \nabla^{t})\|^2_F] + \E[\|(\nabla^{t} \frac{\mathds{1}\mathds{1}^{\top}}{n} - \nabla^{t})\|^2_F]
% \\
% \label{eq:Term-C`1}
% & \leq 2\sigma^2 + \underbrace{2\E[\|(\nabla^{t} \frac{\mathds{1}\mathds{1}^{\top}}{n} - \nabla^{t})\|^2_F]}_{Term(C^{'}_{1})}
% \end{align}
Using $Term(C_{1})$:
\begin{align}
    \E[\|(\nabla^{t} \frac{\mathds{1}\mathds{1}^{\top}}{n} - \nabla^{t})\|^2_F] & = \sum_{i=1}^{n}\E[\|\frac{1}{n}\sum_{j=1}^{n}\nabla f_j(x_{t+\frac{1}{2}}^{(i)}) - \nabla f_i(x_{t+\frac{1}{2}}^{(i)})\|^2_F]
    \\
    & = \sum_{i=1}^{n}\E[\|\frac{1}{n}\sum_{j=1}^{n}\nabla f_j(x_{t+\frac{1}{2}}^{(i)}) - \nabla f_i(x_{t+\frac{1}{2}}^{(i)}) -\nabla f(\Bar{X}_{t+\frac{1}{2}}) + \nabla f(\Bar{X}_{t+\frac{1}{2}})\|^2_F]
    \\
    & \leq 2\sum_{i=1}^{n}\E[\|\frac{1}{n}\sum_{j=1}^{n}(\nabla f_j(x_{t+\frac{1}{2}}^{(i)}) - \nabla f_j(\Bar{X}_{t+\frac{1}{2}}))\|^2_F] + 2\sum_{i=1}^{n}\E[\| \nabla f_i(x_{t+\frac{1}{2}}^{(i)}) - \nabla f(\Bar{X}_{t+\frac{1}{2}})\|^2]
    \\
    & \leq 2L^2 \E[\|\Bar{X}^{t+\frac{1}{2}} - {X}^{t+\frac{1}{2}}\|^2_F]
    + 2\sum_{i=1}^{n}\E[\| \nabla f_i(x_{t+\frac{1}{2}}^{(i)}) - \nabla f(\Bar{X}_{t+\frac{1}{2}}) -\nabla f(x_{t+\frac{1}{2}}^{(i)}) + \nabla f(x_{t+\frac{1}{2}}^{(i)})\|^2]
    \\
    \nonumber
    & \leq 2L^2 \E[\|\Bar{X}^{t+\frac{1}{2}} - {X}^{t+\frac{1}{2}}\|^2_F] + 4\sum_{i=1}^{n}\E[\| \nabla f_i(x_{t+\frac{1}{2}}^{(i)}) -\nabla f(x_{t+\frac{1}{2}}^{(i)}) \|^2] 
    \\&+ 4\sum_{i=1}^{n}\E[\| \nabla f(\Bar{X}_{t+\frac{1}{2}}) -\nabla f(x_{t+\frac{1}{2}}^{(i)})\|^2]
    \\
    & \leq 2L^2 \E[\|\Bar{X}^{t+\frac{1}{2}} - {X}^{t+\frac{1}{2}}\|^2_F] + 4 n B^2 + 4L^2 \E[\|\Bar{X}^{t+\frac{1}{2}} - {X}^{t+\frac{1}{2}}\|^2_F]
    \\
    & \leq 6L^2 \E[\|\Bar{X}^{t+\frac{1}{2}} - {X}^{t+\frac{1}{2}}\|^2_F] + 4 n B^2
    \\
    & \leq 6L^2\E[\|\Bar{X}^{t} + \Bar{\delta}^{t} - ({X}^{t}+ {\delta}^{t})W\|^2_F] + 4nB^2
    \\
    \label{eq:Term-C-1-Solved}
    & \leq 12L^2 (1-\rho)\E[\|\Bar{X}^{t} - {X}^{t}\|^2_F] + 12L^2 (1-\rho)\E[\|\Bar{\delta}^{t} - {\delta}^{t}\|^2_F] + 4nB^2
\end{align}
Now putting $\cref{eq:Term-C-1-Solved}$ in $\cref{eq:C1-start}$, we get
\begin{multline}
\label{eq:consensus-final}
    \E[(C.E)_{t+1}] \leq \Big[(1-\rho)(1+\frac{\rho}{2}) + 12 L^2 \eta^2 (1-\rho)(1+\frac{\rho}{2})\Big]\E[(C.E)_t] + \eta^2 \sigma^2 + 4(1+\frac{2}{\rho})\eta^2 B^2
    \\
    + \Big[(1-\rho) + 12 L^2 \eta^2 (1-\rho)(1+\frac{\rho}{2})\Big]\frac{1}{n}\E[\|\Bar{\delta}^{t} - {\delta}^{t}\|^2_F]
    \\
    \leq \Big[(1-\rho)(1+\frac{\rho}{2}) + 12 L^2 \eta^2 (1-\rho)(1+\frac{\rho}{2})\Big]\E[(C.E)_t] + \eta^2 \sigma^2 + 4(1+\frac{2}{\rho})\eta^2 B^2
    \\
    + \Big[(1-\rho) + 12 L^2 \eta^2 (1-\rho)(1+\frac{\rho}{2})\Big]\frac{2}{n}\Big(\E[\|\Bar{\delta}^{t}\|^2_F] + \E[\|{\delta}^{t}\|^2_F]\Big)
     \\
    \leq \Big[(1-\rho)(1+\frac{\rho}{2}) + 12 L^2 \eta^2 (1-\rho)(1+\frac{\rho}{2})\Big]\E[(C.E)_t] + \eta^2 \sigma^2 + 4(1+\frac{2}{\rho})\eta^2 B^2
    \\
    + \Big[(1-\rho) + 12 L^2 \eta^2 (1-\rho)(1+\frac{\rho}{2})\Big]\frac{2}{n}\Big(\sum_{j=1}^{n} \E\big[\big\| \Bar{\delta}^{t} \|^2\big] + \sum_{j=1}^{n}\E\big[\big\|\delta_{t}^{j} \big\|^2\big]\Big)
    \\
    \leq \Big[(1-\rho)(1+\frac{\rho}{2}) + 12 L^2 \eta^2 (1-\rho)(1+\frac{\rho}{2})\Big]\E[(C.E)_t] + \eta^2 \sigma^2 + 4(1+\frac{2}{\rho})\eta^2 B^2
    \\
    + \Big[(1-\rho) + 12 L^2 \eta^2 (1-\rho)(1+\frac{\rho}{2})\Big]\frac{2}{n}\Big(n \E\big[\big\| \frac{1}{n}\sum_{i=1}^{n}\delta_{t}^{i} \|^2\big] + \sum_{i=1}^{n}{D}^2_{t,i}\Big)
    \\
    \leq \Big[(1-\rho)(1+\frac{\rho}{2}) + 12 L^2 \eta^2 (1-\rho)(1+\frac{\rho}{2})\Big]\E[(C.E)_t] + \eta^2 \sigma^2 + 4(1+\frac{2}{\rho})\eta^2 B^2
    \\
    + \Big[(1-\rho) + 12 L^2 \eta^2 (1-\rho)(1+\frac{\rho}{2})\Big]\frac{2}{n}\Big(\sum_{i=1}^{n}\E\big[\big\|\delta_{t}^{i} \big\|^2\big] + \sum_{i=1}^{n}{D}^2_{t,i}\Big)
    \\
    \leq \Big[(1-\rho)(1+\frac{\rho}{2}) + 12 L^2 \eta^2 (1-\rho)(1+\frac{\rho}{2})\Big]\E[(C.E)_t] + \eta^2 \sigma^2 + 4(1+\frac{2}{\rho})\eta^2 B^2
    \\
    + \Big[(1-\rho) + 12 L^2 \eta^2 (1-\rho)(1+\frac{\rho}{2})\Big]\frac{4}{n}\sum_{i=1}^{n}{D}^2_{t,i}
    \\
    \leq \Big[1 - \frac{\rho}{2} + \frac{24 L^2 \eta^2}{\rho}\Big]\E[(C.E)_t] + \eta^2 \sigma^2 + 4(1+\frac{2}{\rho})\eta^2 B^2
    + \Big[1-\rho + \frac{24 L^2 \eta^2}{\rho}\Big]\frac{4}{n}\sum_{i=1}^{n}{D}^2_{t,i}
\end{multline}
Now, let there be a potential function $\psi_{t}$, defined as,
\begin{equation}
    \label{eq:Prop2-DFL-Potential}
    \psi_{t} = \E[f(\Bar{X}_{t})] + \phi \E[(C.E)_{t}], {\text{  where $\phi$ is a constant.}}
\end{equation}
We now use the potential function to complete the proof.
\begin{flalign}
\nonumber
    \psi_{t+1} - \psi_{t} &= \big\{\E[f(\Bar{X}_{t+1})] - \E[f(\Bar{X}_{t})]\big\} + \phi \big\{\E[(C.E)_{t+1}] - \E[(C.E)_{t}]\big\}
    \\
\nonumber
    &\leq -\frac{\eta}{2}(1 - 2L \eta)\E[\|\Bar{\nabla}_t\|^2]
 + \frac{L\eta^2 \sigma^2}{n} + L^2[\eta(1+2 L \eta) + 2]\frac{1}{n}\sum_{i=1}^nD_{t,i}^2 
 + L^2 \eta (1+2L \eta)\E[(C.E)_t]
 \\ \nonumber
 &+ \phi \Big\{ \Big[1 - \frac{\rho}{2} + \frac{24 L^2 \eta^2}{\rho}\Big]\E[(C.E)_t] + \eta^2 \sigma^2 + 4(1+\frac{2}{\rho})\eta^2 B^2
    + \Big[1-\rho + \frac{24 L^2 \eta^2}{\rho}\Big]\frac{4}{n}\sum_{i=1}^{n}{D}^2_{t,i} - \E[(C.E)_t]\Big\}
    \\
    \nonumber
    &\leq -\frac{\eta}{2}(1 - 2L \eta)\E[\|\Bar{\nabla}_t\|^2] + \eta^2(\phi+\frac{L}{n})\sigma^2 + 4\phi\eta^2(1+\frac{2}{\rho})B^2
    \\
    \nonumber
    &+ \Big(L^2[\eta(1+2 L \eta) + 2] + 4\phi[1-\rho + \frac{24 L^2 \eta^2}{\rho}]\Big)\frac{1}{n}\sum_{i=1}^nD_{t,i}^2 
    \\
    \label{eq:prop-2-potential-phi-value}
    &+ \Big\{L^2 \eta (1+2L \eta) - \phi[\frac{\rho}{2} - \frac{24 L^2 \eta^2}{\rho}]\Big\}\E[(C.E)_t]
\end{flalign}
For $\phi > \frac{L^2 \eta(1+2 L \eta)(1-\rho)}{\frac{\rho}{2} - \frac{24L^2\eta^2}{\rho}}$ and 
$\eta L < \frac{\rho}{4\sqrt{3}} $:
\begin{align}
    \nonumber
    \E[\|\Bar{\nabla}_t\|^2] + \frac{2\Big\{\phi[\frac{\rho}{2} - \frac{24 L^2 \eta^2}{\rho}] - L^2 \eta (1+2L \eta)\Big\}}{\eta(1 - 2L \eta)}\E[(C.E)_t] 
    \leq \frac{2(\psi_{t} - \psi_{t+1})}{\eta(1 - 2L \eta)} + \frac{2\eta(\phi+\frac{L}{n})}{(1 - 2L \eta)}\sigma^2 
    \\
    \label{eq:prop-2-potential-phi-value-1}
    + \frac{8\phi\eta(1+\frac{2}{\rho})B^2}{(1 - 2L \eta)}
    + \frac{2\Big(L^2[\eta(1+2 L \eta) + 2] + 4\phi[1-\rho + \frac{24 L^2 \eta^2}{\rho}]\Big)}{\eta(1 - 2L \eta)}\frac{1}{n}\sum_{i=1}^nD_{t,i}^2 
\end{align}
Now let us assume that $\frac{2\Big\{\phi[\frac{\rho}{2} - \frac{24 L^2 \eta^2}{\rho}] - L^2 \eta (1+2L \eta)\Big\}}{\eta(1 - 2L \eta)}= C$ and we solve for $\phi$:
\begin{align}
    \phi = \frac{\rho(C\eta(1 - 2\eta L) + 2 L^2 \eta(1 + 2\eta L))}{\rho^2 - 48L^2 \eta^2}
\end{align}
So, for $C = L^2$, we get $\phi = \frac{\rho L^2 \eta(3+ 2\eta L)}{\rho^2 - 48L^2 \eta^2}$. Now, summing \cref{eq:prop-2-potential-phi-value-1} for $t=\{1, \ldots ,T\}$ and dividing it by $T$, we get
\begin{align}
    \nonumber
    \frac{1}{T}\sum_{t=1}^T\E[\|\Bar{\nabla}_t\|^2] + \frac{2\Big\{\phi[\frac{\rho}{2} - \frac{24 L^2 \eta^2}{\rho}] - L^2 \eta (1+2L \eta)\Big\}}{\eta(1 - 2L \eta)}\frac{1}{T}\sum_{t=1}^T\E[(C.E)_t] 
    \leq \frac{1}{T}\sum_{t=1}^T\frac{2(\psi_{t} - \psi_{t+1})}{\eta(1 - 2L \eta)} + \frac{2\eta(\phi+\frac{L}{n})}{(1 - 2L \eta)}\sigma^2 
    \\
    \label{eq:prop-2-potential-telescope}
    + \frac{8\phi\eta(1+\frac{2}{\rho})B^2}{(1 - 2L \eta)}
    + \frac{2\Big(L^2[\eta(1+2 L \eta) + 2] + 4\phi[1-\rho + \frac{24 L^2 \eta^2}{\rho}]\Big)}{\eta(1 - 2L \eta)}\frac{1}{nT}\sum_{t=1}^T\sum_{i=1}^nD_{t,i}^2 
\end{align}
Also telescoping over \cref{eq:Prop2-DFL-Potential} for $t=\{1, \ldots ,T\}$ and dividing it by $T$, we get
\begin{equation}
    \label{eq:Prop-2-potential-start-telescope}
    \frac{\psi_{T+1} - \psi_{1}}{T} \geq \frac{f^{*} - f(\Bar{X}_1) - \phi(C.E)_{1}}{T}
\end{equation}
Putting \cref{eq:Prop-2-potential-start-telescope} in \cref{eq:prop-2-potential-telescope} and restructuring it we will get,
\begin{align}
\nonumber
    \frac{1}{T}\sum_{t=1}^T\E[\|\nabla f(\Bar{X}_t)\|^2] +
         \frac{L^2}{T}\sum_{t=1}^T\E[(C.E)_t] 
        \leq \frac{2(f(\Bar{X}_1) -f^{*} + \phi(C.E)_{1})}{\eta(1 - 2\eta L)T} + \frac{2\eta(\phi+\frac{L}{n})}{(1 - 2\eta L)}\sigma^2 
        + \frac{8\phi\eta(1+\frac{2}{\rho})}{(1 - 2\eta L)}B^2
        \\
        + \frac{2\Big(L^2[\eta(1+2 L \eta) + 2] + 4\phi[1-\rho + \frac{24 L^2 \eta^2}{\rho}]\Big)}{\eta(1 - 2\eta L)}\frac{1}{nT}\sum_{t=1}^T \sum_{i=1}^n D_{t,i}^2 
\end{align}
\clearpage
\section{Proof for FedNDL3}
Start:-

\begin{flalign}
\label{eq:Proposed-3-start}
\Bar{X}_{t+1} = \Bar{X}_{t} - \frac{\eta}{n}\sum_{i=1}^{n}(\widetilde \nabla_{t,i} + \delta_{t,i})
\end{flalign}
Now, using L-smoothness, we get
\begin{flalign}
\label{eq:L-smooth-start}
    f(\Bar{X}_{t+1}) & \leq f(\Bar{X}_{t}) + \langle \nabla f(\Bar{X}_{t}),\Bar{X}_{t+1} - \Bar{X}_{t}\rangle + \frac{L}{2} \|\Bar{X}_{t+1} - \Bar{X}_{t}\|^2
    \\
    & \leq f(\Bar{X}_{t}) \underbrace{- \eta \langle \nabla f(\Bar{X}_{t}),\frac{1}{n}\sum_{i=1}^{n}\widetilde \nabla_{t,i}\rangle}_{Term(B)} \underbrace{- \eta \langle \nabla f(\Bar{X}_{t}),\frac{1}{n}\sum_{i=1}^{n}\delta_{t,i}\rangle}_{Term(C)} + \underbrace{\frac{L \eta^2}{2} \|\frac{1}{n}\sum_{i=1}^{n}(\widetilde \nabla_{t,i} + \delta_{t,i})\|^2}_{Term(A)}
\end{flalign}
Term (C):
\\
Taking the expectation of term (C) w.r.t noise, we get
\begin{flalign}
\label{eq:proposed3-C}
    \E[C] = 0
\end{flalign}
Term (B):
\\
Taking the expectation of term (B) w.r.t data, we get
\begin{flalign}
    \E[B] = -\eta \E[\langle \nabla f(\Bar{X}_{t}),\frac{1}{n}\sum_{i=1}^{n}\nabla_{t,i}\rangle]
\end{flalign}
Using the formulation, $\langle a_1, a_2\rangle = \frac{1}{2}\Big(\|a_1\|^2 + \|a_2\|^2 - \|a_1 - a_2\|^2\Big)$
\begin{flalign}
    \E[B] = -\frac{\eta}{2}\E[\|\Bar{\nabla}_t\|^2] - \frac{\eta}{2}\E[\|\frac{1}{n}\sum_{i=1}^{n}\nabla_{t,i}\|^2] + \frac{\eta}{2}\E[\|\Bar{\nabla}_t - \frac{1}{n}\sum_{i=1}^{n}\nabla_{t,i}\|^2]
\end{flalign}
Term (A):
\begin{flalign}
\nonumber
    A &= \frac{L \eta^2}{2} \|\frac{1}{n}\sum_{i=1}^{n}(\widetilde \nabla_{t,i} + \delta_{t,i})\|^2
    \\
    &= \frac{L \eta^2}{2}\Big[ \|\frac{1}{n}\sum_{i=1}^{n}\widetilde \nabla_{t,i}\|^2 +  \|\frac{1}{n}\sum_{i=1}^{n}\delta_{t,i}\|^2 + 2\langle\frac{1}{n}\sum_{i=1}^{n}\widetilde \nabla_{t,i}, \frac{1}{n}\sum_{i=1}^{n}\delta_{t,i}\rangle\Big]
\end{flalign}
Taking the expectation of term (A) w.r.t noise, we get
\begin{flalign}
    \E[A] \leq \frac{L \eta^2}{2}\Big[ \|\frac{1}{n}\sum_{i=1}^{n}\widetilde \nabla_{t,i}\|^2 +  \frac{1}{n}\sum_{i=1}^{n}D^2_{t,i}\Big]
\end{flalign}
Now in the term above taking the expectation w.r.t data, we get
\begin{flalign}
    \E[A] &\leq \frac{L \eta^2}{2}\Big[ \E[\|\frac{1}{n}\sum_{i=1}^{n}\widetilde \nabla_{t,i}\|^2] +  \frac{1}{n}\sum_{i=1}^{n}D^2_{t,i}\Big]
    \\
    & \leq \frac{L \eta^2}{2}\E[\|\frac{1}{n}\sum_{i=1}^{n}\widetilde \nabla_{t,i} - \nabla_{t,i} + \nabla_{t,i}\|^2] + \frac{L \eta^2}{2}\frac{1}{n}\sum_{i=1}^{n}D^2_{t,i}
    \\
    & \leq \frac{L \eta^2}{2}\E\Big[\|\frac{1}{n}\sum_{i=1}^{n}\widetilde \nabla_{t,i} - \nabla_{t,i} \|^2 + \|\frac{1}{n}\sum_{i=1}^{n}\nabla_{t,i}\|^2 + 2\langle\frac{1}{n}\sum_{i=1}^{n}\widetilde \nabla_{t,i} - \nabla_{t,i},\frac{1}{n}\sum_{i=1}^{n}\nabla_{t,i}\rangle\Big] + \frac{L \eta^2}{2}\frac{1}{n}\sum_{i=1}^{n}D^2_{t,i}
    \\
    & \leq \frac{L \eta^2}{2}\Big(\frac{\sigma^2}{n} + \E[\|\frac{1}{n}\sum_{i=1}^{n}\nabla_{t,i}\|^2]\Big) + \frac{L \eta^2}{2}\frac{1}{n}\sum_{i=1}^{n}D^2_{t,i}
\end{flalign}
Now putting all the A, B, and C terms in \cref{eq:L-smooth-start}, we get
\begin{flalign}
\nonumber
    \E[f(\Bar{X}_{t+1})] &\leq \E[f(\Bar{X}_{t})] - \frac{\eta}{2}\E[\|\Bar{\nabla}_t\|^2]  \underbrace{- \frac{\eta}{2}\E[\|\frac{1}{n}\sum_{i=1}^{n}\nabla_{t,i}\|^2]}_{Term (M)} + \frac{\eta}{2}\E[\|\Bar{\nabla}_t - \frac{1}{n}\sum_{i=1}^{n}\nabla_{t,i}\|^2] + \frac{L \eta^2}{2}\frac{\sigma^2}{n} 
    \\
    \label{eq:A+B+C}
    &+ \underbrace{\frac{L \eta^2}{2}\E[\|\frac{1}{n}\sum_{i=1}^{n}\nabla_{t,i}\|^2]}_{Term (N)} + \frac{L \eta^2}{2}\frac{1}{n}\sum_{i=1}^{n}D^2_{t,i}
\end{flalign}
Dropping Term (M) and then working with Term (N).
\begin{flalign}
    N & = \frac{L \eta^2}{2}\E[\|\frac{1}{n}\sum_{i=1}^{n}\nabla_{t,i}\|^2]
    \\ & = \frac{L \eta^2}{2}\E[\|\frac{1}{n}\sum_{i=1}^{n}\nabla_{t,i} - \Bar{\nabla}_t + \Bar{\nabla}_t\|^2]
    \\
    & \leq L \eta^2 \E[\|\frac{1}{n}\sum_{i=1}^{n}\nabla_{t,i} - \Bar{\nabla}_t \|^2] + L \eta^2\E[\|\Bar{\nabla}_t\|^2]
\end{flalign}
Now putting back the M and N terms in \cref{eq:A+B+C}, we get
\begin{flalign}
\nonumber
    \E[f(\Bar{X}_{t+1})] &\leq \E[f(\Bar{X}_{t})] - \frac{\eta}{2}\E[\|\Bar{\nabla}_t\|^2] + \frac{\eta}{2}\E[\|\Bar{\nabla}_t - \frac{1}{n}\sum_{i=1}^{n}\nabla_{t,i}\|^2] + \frac{L \eta^2}{2}\frac{\sigma^2}{n} + L \eta^2 \E[\|\frac{1}{n}\sum_{i=1}^{n}\nabla_{t,i} - \Bar{\nabla}_t \|^2] 
    \\
    &+ L \eta^2\E[\|\Bar{\nabla}_t\|^2] + \frac{L \eta^2}{2}\frac{1}{n}\sum_{i=1}^{n}D^2_{t,i}
    \\ \label{eq:Term-O-start}
    & \leq \E[f(\Bar{X}_{t})] - \frac{\eta}{2}(1 - 2L \eta)\E[\|\Bar{\nabla}_t\|^2] + \frac{L \eta^2}{2}\frac{\sigma^2}{n} 
    + \frac{\eta}{2}(1 + 2 L \eta)\underbrace{\E[\|\frac{1}{n}\sum_{i=1}^{n}\nabla_{t,i} - \Bar{\nabla}_t \|^2]}_{Term (O)} 
    + \frac{L \eta^2}{2}\frac{1}{n}\sum_{i=1}^{n}D^2_{t,i}
\end{flalign}
Using Term (O):
\begin{flalign}
    O & = \E[\|\frac{1}{n}\sum_{i=1}^{n}\nabla_{t,i} - \Bar{\nabla}_t \|^2]
    \\
    & \leq \frac{L^2}{n}\sum_{i=1}^{n}\E[\|X_{t,i} - \Bar{X}_t\|^2]
    \\
    \label{eq:Term-O-end}
    & \leq \frac{L^2}{n}\E[\|X^t - \Bar{X}^t\|^2_F]
\end{flalign}
Putting \cref{eq:Term-O-end} in \cref{eq:Term-O-start} and also denoting $\frac{1}{n}\|\Bar{X}^t - X^{t}\|_{F}^2 = (C.E)_t$, we get
\begin{flalign}
\label{eq:L-smooth-final}
    \E[f(\Bar{X}_{t+1})] &\leq \E[f(\Bar{X}_{t})] - \frac{\eta}{2}(1 - 2L \eta)\E[\|\Bar{\nabla}_t\|^2] + \frac{L \eta^2}{2}\frac{\sigma^2}{n} + \frac{L \eta^2}{2}\frac{1}{n}\sum_{i=1}^{n}D^2_{t,i} 
    + \frac{L^2\eta}{2}(1+2L\eta)\E[(C.E)_t]
\end{flalign}
Using the assumption on consensus error at time t+1,
\begin{flalign}
    \frac{1}{n}\E[\|X^{t+1} - \Bar{X}^{t+1}\|^2_F] &\leq \rho_t \frac{1}{n}\E[\|X^t - \Bar{X}^t\|^2_F] + \gamma_t
    \\
    \E[(C.E)_{t+1}] &\leq \rho_t \E[(C.E)_{t}]] + \gamma_t
\end{flalign}
Let there be a potential function $\psi_{t}$, defined as,
\begin{equation}
    \label{eq:DFL-Potential}
    \psi_{t} = \E[f(\Bar{X}_{t})] + \phi_t \E[(C.E)_{t}], {\text{  where $\phi_t > 0$.}}
\end{equation}
Now, we use the potential function to complete the proof.
\begin{multline}
    \psi_{t+1} - \psi_{t} = \big\{\E[f(\Bar{X}_{t+1})] - \E[f(\Bar{X}_{t})]\big\} + \phi_{t+1} \E[(C.E)_{t+1}] - \phi_{t}\E[(C.E)_{t}]
    \\
    \leq - \frac{\eta}{2}(1 - 2L \eta)\E[\|\Bar{\nabla}_t\|^2] + \frac{L \eta^2}{2}\frac{\sigma^2}{n} + \frac{L \eta^2}{2}\frac{1}{n}\sum_{i=1}^{n}D^2_{t,i} + \frac{L^2\eta}{2}(1+2L\eta)\E[(C.E)_t] 
    + \phi_{t+1} \rho_t \E[(C.E)_{t}] + \phi_{t+1} \gamma_t - \phi_{t} \E[(C.E)_{t}]
    \\
    \label{eq:potential-start}
    \leq - \frac{\eta}{2}(1 - 2L \eta)\E[\|\Bar{\nabla}_t\|^2] + \frac{L \eta^2}{2}\frac{\sigma^2}{n} + \frac{L \eta^2}{2}\frac{1}{n}\sum_{i=1}^{n}D^2_{t,i} + \phi_{t+1} \gamma_t + \Big(\frac{L^2\eta}{2}(1+2L\eta)+ \phi_{t+1} \rho_t - \phi_{t} \Big) \E[(C.E)_{t}]
\end{multline}
Pick $\phi_{t}$ such that
\begin{flalign}
\nonumber
    \phi_{t} & > \frac{L^2\eta}{2}(1+2L\eta)+ \phi_{t+1} \rho_t
\end{flalign}
So, let
Using $\phi_{t}  = L^2\eta(1+2L\eta)+ 2\phi_{t+1} \rho_t$ and $\eta L < \frac{1}{2}$ in \cref{eq:potential-start}, we get
\begin{flalign}
\nonumber
  \E[\|\Bar{\nabla}_t\|^2] +\frac{\Big(L^2\eta(1+2L\eta)+ 2\phi_{t+1} \rho_t\Big)}{\eta(1 - 2L \eta)} \E[(C.E)_{t}]\leq \frac{2(\psi_{t} - \psi_{t+1})}{\eta(1 - 2L \eta)} + \frac{L}{(1 - 2L \eta)}\frac{\sigma^2}{n} 
   \\
   \label{eq:potential-phi-value}
   + \frac{L \eta}{(1 - 2L \eta)}\frac{1}{n}\sum_{i=1}^{n}D^2_{t,i} + \frac{2\phi_{t+1}\gamma_t}{\eta(1 - 2L \eta)} 
\end{flalign}
Let $C = \frac{L^2\eta(1+2L\eta)+ 2\phi_{t+1} \rho_t}{\eta(1 - 2L \eta)} $. Solving for $\phi_{t+1}$:
\begin{align}
    \phi_{t+1} = \frac{C\eta(1 - 2\eta L) - L^2 \eta(1+2\eta L)}{2\rho_t}
\end{align}
So, for $C = 2L^2$ and $\eta L < \frac{1}{6}$, we have $\phi_{t+1} = \frac{L^2\eta(1 - 6 \eta L)}{2\rho_t}$. Now, summing \cref{eq:potential-phi-value} for $t=\{1, \ldots, T\}$ and dividing it by $T$, we get
\begin{flalign}
\nonumber
  \frac{1}{T}\sum_{t=1}^T\E[\|\Bar{\nabla}_t\|^2] +\frac{1}{T}\sum_{t=1}^T\frac{\Big(L^2\eta(1+2L\eta)+ 2\phi_{t+1} \rho_t\Big)}{\eta(1 - 2L \eta)} \E[(C.E)_{t}]\leq \frac{2\frac{1}{T}\sum_{t=1}^T(\psi_{t} - \psi_{t+1})}{\eta(1 - 2L \eta)} + \frac{L}{(1 - 2L \eta)}\frac{\sigma^2}{n} 
   \\
   \label{eq:potential-end}
   + \frac{2\frac{1}{T}\sum_{t=1}^T\phi_{t+1}\gamma_t}{\eta(1 - 2L \eta)}+ \frac{L \eta}{(1 - 2L \eta)}\frac{1}{nT}\sum_{t=1}^T\sum_{i=1}^{n}D^2_{t,i}  
\end{flalign}
Also telescoping over \cref{eq:DFL-Potential} for $t=\{1 \cdots T\}$ and dividing it by $T$, we get
\begin{equation}
    \label{eq:proposed3-potential-start-telescope}
    \frac{\psi_{T+1} - \psi_{1}}{T} \geq \frac{f^{*} - f(\Bar{X}_1) - \phi_1(C.E)_{1}}{T}
\end{equation}
Putting \cref{eq:proposed3-potential-start-telescope} in \cref{eq:potential-end} and restructuring it we will get,
\begin{multline}
        \frac{1}{T}\sum_{t=1}^{T}\E[\|\nabla f(\Bar{X}_t)\|^2] + \frac{2L^2}{T}\sum_{t=1}^{T}\E[(C.E)_t] 
        \leq \frac{2(f(\Bar{X}_1) - f^{*} + \phi_1(C.E)_{1})}{\eta (1 - 2L \eta) T} + \frac{L \eta}{n(1 - 2 L \eta)}\sigma^2 \\ + \frac{L^2(1-6L\eta)}{T(1-2L \eta)}\sum_{t=1}^{T}\frac{\gamma_t}{\rho_{t}}
     + \frac{L \eta}{(1-2 L \eta)}\frac{1}{nT}\sum_{t=1}^{T}\sum_{i=1}^{n}{D}^2_{t,i}
    \end{multline}
\end{document}